\documentclass[11pt]{article}
\usepackage{amsfonts,amsthm,amsmath,amssymb,mathtools}

\usepackage{fullpage}
\usepackage{thmtools}
\usepackage{hyperref}

\usepackage[sort&compress]{natbib}
\usepackage{algorithm,algorithmic}
\usepackage{subcaption}
\usepackage{graphicx}
\usepackage{color}
\usepackage[export]{adjustbox}

\usepackage{wrapfig}
\usepackage{bbm}
\DeclareMathOperator*{\argmin}{arg\,min}

\usepackage{xcolor}

\newcommand{\jltext}[1]{{\color{cyan} #1}}
\colorlet{cyan}{black}
\allowdisplaybreaks[4]

\usepackage{url}

\usepackage[figuresright]{rotating}

\newcounter{thm_counter}

\newcounter{pro_counter}

\newtheorem{theorem}[thm_counter]{Theorem}
\newtheorem{lemma}[thm_counter]{Lemma}

\newtheorem{proposition}[pro_counter]{Proposition}

\newcommand{\cA}{\mathcal{A}}
\newcommand{\cB}{\mathcal{B}}

\newcommand{\cE}{\mathcal{E}}
\newcommand{\cL}{\mathcal{L}}

\newcommand{\E}{\mathbb{E}}

\newcommand{\bP}{\mathbb{P}} 
\newcommand{\ds}{\displaystyle}

\def\TEV{EVT}
\def\APTEV{AP-EVT}
\def\TEVf{EVT$_{pf}$}
\def\APTEVf{AP-EVT$_{pf}$}

\def\rc{\color{red}} 
 
\def\rcc{\color{red}}

\def\E{\mathbb{E}}
\def\R{\mathbb{R}}

\def\rc{\color{red}} 
 
\def\rcc{\color{red}}
\def\rcc{}

\def\rcc{}

\usepackage{enumitem}

\title{Asynchronous Parallel Empirical Variance Guided Algorithms for the Thresholding Bandit Problem}
\author{Jie Zhong$^\dag$, Yijun Huang$^\ddag$, and Ji Liu$^{\ddag,\flat,\natural}$\\
\{jiezhongmath, huangyj0, ji.liu.uwisc\}@gmail.com\\
$^\dag$Department of Mathematics, University of Rochester\\
$^\ddag$Department of Computer Science, University of Rochester\\
$^\flat$Department of Electrical Computer Engineering, University of Rochester \\
$^\natural$Goergen Institute for Data Science, University of Rochester
}
\date{\today}

\begin{document}
\maketitle
\begin{abstract}
  This paper considers the multi-armed thresholding bandit problem -- identifying
all arms whose expected rewards are above a predefined threshold via as few pulls (or rounds) as possible
-- proposed by \citet{locatelli2016optimal} recently. Although the proposed
algorithm in \citet{locatelli2016optimal} achieves the optimal round
complexity\footnote{The round complexity in this paper means the number of
  rounds required to guarantee certain accuracy of identifying arms.} in a
certain sense, there still remain unsolved issues. This paper proposes an
asynchronous parallel thresholding algorithm and its parameter-free version to
improve the efficiency and the applicability. On one hand, the proposed two
algorithms use the empirical variance to guide \jltext{the pull decision} at each round,
and significantly improve the round complexity of the ``optimal'' algorithm when all arms have
bounded high order moments. \jltext{The proposed algorithms can be proven to be optimal.} On the other hand, most bandit algorithms assume that the reward can be observed immediately after the pull or the next decision would not be made before all rewards are observed.  Our proposed asynchronous parallel algorithms allow making the choice of the next pull with unobserved rewards from earlier pulls, which avoids such an unrealistic assumption and significantly improves the identification process. Our theoretical analysis justifies the effectiveness and the efficiency of proposed asynchronous parallel algorithms.
The empirical study is also provided to validate the proposed algorithms.
\end{abstract}

\section{Introduction}

The multi-armed bandit (MAB) problem is the most basic and important model for
sequential decision theory. Since its very first study in the clinical
trials \citep{thompson1933likelihood}, the MAB problem has been received extensive attention and applied to 
various different fields: online advertisement \citep{li2010contextual}, hyperparameter tuning
\citep{agarwal2011oracle,jamieson2016non}, network
routing \citep{awerbuch2004adaptive}, portfolio selection \citep{sani2012risk}, etc.

Key concerns in the MAB problem include
\begin{itemize}
\item Maximizing the expected total discounted reward
  \citep{bellman1956problem, gittins1979dynamic}, or minimizing the regret for the total expected return \citep{lai1985asymptotically, auer2002finite, bubeck2012regret},
\item Best arm identification, Best $K$-arm identification \citep{audibert2010best,gabillon2012best,kaufmann2015complexity,cao2015top},
\item Thresholding (or threshold) problem \citep{ma2014active, ma2015active,locatelli2016optimal,abernethy2016threshold}.
\end{itemize}

In this paper, we consider the thresholding problem, which summarizes many important real applications such as Dark Pool brokerage \citep{ganchev2010censored,agarwal2010optimal, amin2012budget}, active anomaly
detection \citep{steinwart2005classification}, and active binary classification \citep{tong2001support}. It can be formally described in the following: 

Given $K$ arms with bounded random rewards. When an arm is pulled, it generates a reward, which follows an unknown distribution (WLOG within $[0,1]$).  Let $b\in \R$ be a predefined threshold. The goal of the thresholding bandit problem is to identify arms whose expected rewards are above the threshold (or equivalently below the threshold) via as few pulls (or budget, more
generally) as possible, after making sequential decisions on which arm to pull in each round.



{\rcc To the best of our knowledge, so far the only algorithm for the thresholding problem is proposed by \citet{locatelli2016optimal}, namely any time pull (ATP) algorithm.} There still remain some limitation and important issues unsolved:
\begin{itemize}
\item {\rcc The ATP algorithm only uses the first moment estimate (that is the empirical mean) to make the
  decision at each round without using any high order moment estimate for
  distribution $\nu_k$. This motivates people to ask: can we do better via empirical estimation for high order moments?}
\item {\rcc The ATP algorithm (as well as most existing bandit algorithms) cannot decide the next pull before all rewards are observed,} which may significantly degrade the efficiency in many real applications especially where it takes a long time to wait for the reward after an arm gets pulled. For example, the blood test may take long time to get the result in clinical trials and training the deep neural network is very time consuming after the hyperparameter is set. 
\end{itemize} 

This paper makes the first attempt to treat these two issues for the thresholding bandit problem. More specifically, we first propose an {\rc  E}mpirical {\rc V}ariance guided {\rc  T}hresholding (EVT) algorithm to improve the ``optimal'' algorithm by \citet{locatelli2016optimal}, {\rcc via utilizing both empirical estimates for the first moment
and second moment (variance)}. Both our theoretical analysis
and experiments show that the round complexity by EVT
could be significantly less than ATP \citep{locatelli2016optimal} to achieve the same confidence level
especially when the variance is small. \jltext{The proposed algorithms can also proven to be optimal.} To address the second issue, we propose a
parallel version for the EVT algorithm, namely {\rc A}synchronous {\rc P}arallel {\rc
  E}mpirical {\rc V}ariance guided {\rc  T}hresholding (AP-EVT) algorithm. This algorithm allows us to make the decision at each round with unobserved rewards. It {\rcc is} also applicable to the multi-agent scenario that multiple
agents cooperate and work in parallel in the asynchronous fashion: {\rcc all agents run the procedure concurrently -- any agent makes its own decision on which arm to pull next based on rewards from all agents \emph{without} waiting for the pending pulls.\footnote{A pending pull means that the reward has not been observed after the arm gets pulled by any agent.}}
Our theoretical analysis
suggests that the round complexity will not degrade significantly as long as the
number of pending (or unobserved) rewards does not dominate. Moreover, we also propose parameter free
versions of our algorithms: EVT$_{pf}$ and AP-EVT$_{pf}$, which enjoy similar
theoretical guarantees to parameter dependent algorithms, {\rcc but \emph{without} assuming to know the total number of rounds or any other parameters in advance. Our
empirical study shows} that AP-EVT (or EVT) slightly outperforms AP-EVT$_{pf}$ (or
EVT$_{pf}$), due to the accessibility to additional information. The main
results of this paper can be summarized in Table~\ref{tab:summary}.
\begin{table*}
\caption{The required number of rounds to achieve the confidence level
  $1-\epsilon$. {\rcc $\delta\in [0,1]$ is the algorithm parameter. $\eta\geq 0$ is the upper bound of ratio of the number of unobserved rewards over the number of observed rewards over all arms. $\Delta_k$ is between $0$ and $1$ under the assumption that all arms are within $[0,1]$.}
  } \label{tab:summary}
\begin{center}
  {
    \renewcommand{\arraystretch}{1.5}
    \resizebox{\columnwidth}{!}{
\begin{tabular} {l | c | c }
\hline
 &  Rounds Complexity  & Problem Complexity $H$ \\
 \hline \hline
  ATP \citep{locatelli2016optimal} & $\Theta(H(\log(nK) + \log(1/\epsilon)))$ & $\sum_{k=1}^K\Delta_k^{-2}$\\
  EVT (and EVT$_{pf}$) [Theorem \ref{thm:EVT}]& $\Theta(H(\log(nK) + \log(1/\epsilon)))$ & $\sum_{k=1}^K (\sigma_k^2\Delta_k^{-2} + \Delta_k^{-1})$ \\
  AP-EVT [Theorem \ref{thm:AP-EVT}] & $\Theta(H(\log(nK) + \log(1/\epsilon)) + (1-\delta)
           \tau)$ & $(1+\delta\eta)^2\sum_{k=1}^K (\sigma_k^2\Delta_k^{-2} + \Delta_k^{-1})$ \\
  AP-EVT$_{pf}$ [Theorem \ref{thm:AP-EVT-pf}] & $\Theta(H(\log(nK) + \log(1/\epsilon)) + (1-\delta)
                  \tau)$ & $(1+\delta\eta)\sum_{k=1}^K (\sigma_k^2\Delta_k^{-2} + \Delta_k^{-1})$ \\
\hline
\end{tabular}
}
}
\end{center}
\end{table*}
The main contribution of this paper can be summarized in the following
\begin{itemize}
\item To the best of our knowledge, the proposed {\rcc \TEV~as well as its parameter free version \TEVf}~is the first thresholding
  bandit algorithm using empirical variance. The theoretical analysis provides
  interesting insights when and why using empirical variance could significantly improve the optimal methods only using the first moment estimate.
\item {\rcc The proposed algorithms \APTEV~and \APTEVf~allow the asynchronous\footnote{The synchronous parallel implementation can be considered as a special case of asynchronous parallel implementation.} parallel implementation which can significantly improve the \jltext{parallel efficiency and make the thresholding algorithms scalable to large scale problems.} This idea in \APTEV~and \APTEVf~can be applied to other bandit algorithms such as UCB. The theoretical analysis also reveals interesting clues how the unobserved rewards affect the performance.}
\end{itemize}

\paragraph{Notations and definitions}
Throughout the whole paper, we use the following notations:
\begin{itemize}
\item $n$: The total number of rounds in the problem. 
\item $K$: The total number of arms; $[K]$: The set of $K$ arms, that is, $\{1,2,\cdots, K\}$.
\item $b$: The predefined threshold.
\item $\nu_k$: The unknown distribution for the reward of arm $k$, with mean $\mu_k$ and variance $\sigma_k^2$.
\item $U_b := \{k\in [K]: \mu_k \ge b\}$, and $U_b^C$ as its complement set, that is, $U_b^C = \{k \in [K]: \mu_k < b\}$.
\item $\Delta_i := |\mu_i - b|$, that is, the gap between the mean and the threshold.
\end{itemize}
\paragraph{Organization}
The rest of this paper is organized as follows:
Section 2 summarizes the related work; 
Section 3 describes algorithms and main results; Experimental setups and results are contained in Section 4; and we conclude with Section 5. All proofs
are provided in the supplementary material.

\section{Related Works}
There are many different research directions in multi-armed bandit problems. For
example, the early work of \citet{bellman1956problem} aims to maximize the
expected total discounted reward over an infinite time horizon. \citet{lai1985asymptotically}
analyzes the regret for the expected total reward over a finite horizon.
Thanks to the wide applications, recently other variants of bandit problems are studied: dueling bandits \citep{ramamohan2016dueling,wu2016double}; bandit
optimization problem \citep{yang2016optimistic,bubeck2016kernel}; cascading
bandits \citep{li2016contextual,lagree2016multiple}, to name a few.

Our work follows the line of research on stochastic multi-bandit problem in finding the optimal object. \citet{audibert2010best}, in particular, consider the problem of the best arm identification, and prove that the probability of error of UCB-E strategy is at most of $\exp(-n/H)$, where $H$ is a characterization of hardness of the problem. That is, $H$ gives the order of total number of samples required to find the best arm with high probability.
\citet{locatelli2016optimal} study the thresholding bandit problem and close, up
to constants, the gap in the finite budget setting, where the complexity $H$ is
the same as in the best arm identification problem.


Most of work mentioned earlier does not take into account of the analysis of
empirical variance. \citet{audibert2009exploration} show that the algorithm that
uses the variance estimates has a major advantage over its alternatives that do
not use such estimates provided that the variances of the payoffs of the
suboptimal arms are low. \citet{degenne2016combinatorial} study the
combinatorial semi-bandit problem with covariance, but known. In this direction,
the work closest to ours is a recent work by \citet{gabillon2011multi}, which
studies the problem of best-arm identification in a multi-armed multi-bandit
setting under a fixed budget constraint and takes into account the variance of
the arms. However, their algorithm requires the pre-knowledge of the problem
complexity $H$, which is in practice unknown. Instead, in Section \ref{sec:pf}
we propose a parameter free version of variance guided algorithm and obtain a
similar and comparable round complexity. 

To the best of our knowledge, this paper is the first to analyze the
speedup of asynchronous parallel algorithms for the thresholding bandit problem. The idea of using asynchronous parallelism (AP) to avoid idling any child worker appeared in 1980s \citep{bertsekas1989parallel}, but recently received remarkable success and broad attention in machine learning and optimization communities. Note that in most early literature's, AP is used to parallelize optimization based iterative algorithms. For example, \citet{AA11a}, \citet{FN11a}, and \citet{lian2015asynchronous} applied the asynchronous parallelism to accelerate the stochastic gradient descent algorithm for solving deep learning, primal SVM, matrix completion, etc. \citet{JL14a} and \citet{HYD15} proposed the asynchronous parallel stochastic coordinate descent for solving dual SVM, LASSO, etc. \citet{HA14a} used the asynchronous parallelism to accelerate solving the linear system. 

\section{Algorithms and Results}

To introduce the proposed algorithm and illustrate its connection to other existing
algorithms, we define the generic framework summarized in Algorithm~\ref{alg:GF}.
Step 1 serves to the initialization purpose to ensure all arms have been pulled
for a certain number of times, which is usually required in bandit algorithms.
The key step is Step 4 -- making the sequential decision which arm to pull next
based on the current observed rewards.

Specifically, let's define $\hat{\mu}_{k, t}$ to be the empirical mean for arm
$k$ with $t$ observed rewards, i.e.,
\begin{align}
\hat{\mu}_{k, t} =  {1\over t}\sum_{i=1}^t X_{k,i},
\end{align}
where $X_{k,i}$ is the reward observed when pulling arm $k$ for the
$i$-th time. Define $T_k(t)$ to be the number of rewards observed up to round $t$ for arm $k$.

Step 4 chooses the arm to pull based on the quantity for each arm
\[
|\hat{\mu}_{k, T_k(t)}- b| S_k(t),
\]
that is, the arm with smallest such quantity will be selected. This quantity includes two components: 
\begin{itemize}
\item $|\hat{\mu}_{k, T_k(t)}- b|$ estimates the difference between the empirical mean and the threshold. The smaller, the more chances it will be selected, since it is harder to distinguish.
\item $S_k$ essentially measures the confidence on the current estimation. It usually increases w.r.t. the number of observed rewards. 
\end{itemize}
\begin{algorithm}[H]
\begin{algorithmic}[1]
\REQUIRE $b$, $n$ and $a$
\ENSURE $\hat{U}_b(n)=\{k\in [K]~|~\hat{\mu}_{k, T_k(n)} \geq b\}$
\STATE Pull all $K$ arms twice and observed rewards $\{X_{k,1}\}_{k=1}^K$ and $\{X_{k,2}\}_{k=1}^K$
\FOR{$t=2K:n$}
\STATE Calculate the number of observed rewards $T_k(t)$ at the current time point $t$ for all arms $k=1,2,\cdots, K$ 
\STATE Select arm 
\[
I_{t+1} =\argmin_k \quad |\hat{\mu}_{k,T_k(t)}-b| S_k(t)
\]
\STATE Send out the order to pull arm $I_{t+1}$ 
\ENDFOR
\end{algorithmic}
\caption{A General Framework of Thresholding Multi-Bandit} \label{alg:GF}
\end{algorithm}

\paragraph{Any Time Pull (ATP) Algorithm} The ATP algorithm \citep{locatelli2016optimal} chooses $S_k(t)$ to be
\[
(\text{ATP})\quad S_k(t) = \sqrt{T_k(t)},
\]
that is, the confidence on any arm increases w.r.t. the square root of the number of observed rewards. This algorithm is proven to be the optimal with the following guarantee:
\begin{theorem} [\citet{locatelli2016optimal}, Theorem 2]
Let $K>0, n> 2K$, and consider a thresholding bandit problem. Assume that all
arms of the problem are bounded in $[0,1]$. Given $\epsilon\in (0,1)$, if
\begin{align}
  n \geq \Theta\left(H_{\text{ATP}} (\log (nK) + \log (1/\epsilon)) \right), \label{eq:ATP_complexity}
\end{align}
where $\Theta$ is some positive constant and $H_{\text{ATP}}$ is defined as
\[
  H_{\text{ATP}} : = \sum_{i=1}^K \Delta_i^{-2},
\]
then the Algorithm ATP guarantees that with probability at least $1-\epsilon$, the
player can correctly discriminate arms in $U_b$ from those in $U_b^C$, that is,
$\hat{U}_b(n)= U_b$.

\end{theorem}

\subsection{Empirical Variance Guided Thresholding (EVT)} 
The previous ATP algorithm does not consider the variance for each arm. Intuitively, the variance information could be very useful. For example, assuming that arm $k$ is deterministic (or its variance is zero), one only needs to pull it once, no matter how its mean or empirical mean is close to the threshold. 
This motivates us to use the variance information to guide us selecting the arm in each round.

Define $\hat{\sigma}_{k,t}$ to be the empirical variance for arm $k$ with $t$ observed rewards
\begin{align}\label{eq:var}
\hat{\sigma}_{k,t} = & \sqrt{{1\over t} \sum_{i=1}^t(X_{k,i} - \hat{\mu}_{k,t})^2}.
\end{align}

The proposed EVT algorithm uses \eqref{eq:var} to estimate
the variance and chooses $S_k(t)$ to depend on both the number of observed
rewards and the empirical variance, i.e.,
\begin{equation}
  \label{eq:S_k}
  (\text{EVT})\quad S_k(t)=\left({a\over T_{k}(t)} + \sqrt{a\over T_{k}(t)}
  \hat{\sigma}_{k,T_k(t)}\right)^{-1}，
\end{equation}
where $a$ is a predefined parameter. As a result, arms with low empirical
variance in the proposed EVT algorithm need to be pulled much less than them in
the ATP algorithm so that the round complexity is reduced, which can be
indicated by the following theorem.

 \begin{theorem} \label{thm:EVT} 
Let $K>0, n> 2K$, and consider a thresholding bandit problem. Assume that all
arms of the problem are bounded in $[0,1]$. Let $a = n/H_{\text{EVT}}$ in \eqref{eq:S_k} with
\[
  H_{\text{EVT}} : = \sum_{i=1}^K (\sigma_i^2\Delta_i^{-2} + \Delta_i^{-1}).
\]
Given $\epsilon\in (0,1)$, if
\begin{align}
  n \geq \Theta[H_{\text{EVT}} (\log (nK)  +\log (1/\epsilon)) ],
  \label{eq:EVT_complexity}
\end{align}
where $\Theta$ is some positive constant, then the EVT Algorithm guarantees that with probability at least $1-\epsilon$, the
player can correctly discriminate arms in $U_b$ from those in $U_b^C$, or equivalently, 
\begin{equation}
  \label{eq:upperbound}
\mathbb{P}(\hat{U}_b(n) \neq U_b) \leq \exp \left(-\Theta\left({n\over H_{\text{EVT}}}+\log (nK)\right)\right).
\end{equation}
\end{theorem}

Comparing to the ATP algorithm, the required rounds (or pulls) $n$ mainly differs at the
definition of $H$. We highlight the following key observations (since $\sigma_i,
\Delta_i\in [0,1]$):
\begin{itemize}
\item It follows from Lemma \ref{lem:compareH} (see Appendix) that
\[
  \sigma^2_i \Delta_i^{-2} + \Delta_i^{-1} \le \Delta_i^{-2},
\]
and thus
\[
  H_{\text{EVT}} \le H_{\text{ATP}}.
 \]
\item Furthermore, in the case that variances of arms are very small (or $\sigma_i
  \ll 1$), $H_{\text{EVT}}$ can be significantly smaller (better) than $H_{\text{ATP}}$.
\end{itemize}
\jltext{
  Acute readers may notice that the round complexity in \eqref{eq:ATP_complexity} by \citet{locatelli2016optimal} is claimed to be optimal in a certain sense. Then one may ask
  that if the arm variances are small, does our bound \eqref{eq:EVT_complexity} violate the optimal bound? A short answer is that our bound can be significantly smaller (or better) than the bound in
  \eqref{eq:ATP_complexity}, but does not violate its optimality! In fact, the optimality in \eqref{eq:ATP_complexity} is restricted in the family of the threshold bandit problems with complexity characterized by $H_{\text{ATP}}$ only depending on the means. Our bound \eqref{eq:EVT_complexity} considers the family with complexity characterized by $H_{\text{ETV}}$ depending on both the means and the variances, which is different from $H_{\text{ATP}}$ but characterizes more subtle structures in the bandit problem. The following theorem will show that our bound is tight and optimal (up to a constant factor) within the family defined on $H_{\text{EVT}}$.
\begin{theorem}
  \label{thm:lowerbound}
For the thresholding bandit problem, we have the following lower bound
  \[
  \inf_{A\in \mathcal{A}}\sup_{D\in \mathcal{D}_{\text{EVT}}(h)} \mathbb{P}^{A}_{D}(\hat{U}_b(n) \neq U_b) \geq \exp\left(-{10n\over h}- 16\log(5nK)\right)
  \]
where $\mathcal{A}$ is the set including all possible algorithms,
$\mathcal{D}_{\text{EVT}}(h)$ is defined to be the set of distributions of arms with the
$H_{\text{EVT}}$ value being $h$, $\mathbb{P}^A_D$ denotes the probability
of applying algorithm $A$ on distribution $D$, and $\hat{U}_b(n)$ is the output of an algorithm after pulling $n$ rounds of
arms.
\end{theorem}  
This theorem basically indicates that for any thresholding bandit algorithm,
there always exists a distribution such that the probability of making a mistake is greater than
  \[
    \exp\left(-{10n\over H_{\text{EVT}}}- 16\log(5nK)\right).
  \]
It suggests that our algorithm is tight and optimal up to a constant factor.
}


\subsection{Asynchronous Parallel Empirical Variance Guided Thresholding
  (AP-EVT)}

We consider the asynchronous parallel version of the proposed EVT algorithm. Most
multi-armed bandit algorithms make the assumption that the reward can be
observed immediately right after the arm is pulled or one does not decide new pull before observe all pulls. This is unrealistic or time consuming in many applications. For example, for the
hyperparameter tuning in deep learning, each arm associates with a
setup of all hyperparameters (for example, number of layers, number of nodes in each layer, and choice of the activation function at each layer); pulling an arm means training the neural network using
this hyperparameter setup; and the reward is the accuracy on the
variation data set using the trained neural network model under this hyperparameter setup. As we know, it is really time consuming for the training process, especially when training on
a large dataset. So it would be very inefficient if one has to wait for the outcome of the pulled arm to decide the next pull (train). Our proposed algorithm allows one to decide the next arm to pull (or the next setup of hyperparameters to train) only based on current observed rewards, without waiting for unobserved rewards. From the perspective of arms, typically multiple arms are running concurrently or in parallel asynchronously (since no arm has to wait for the rewards from other arms).




We use $T_k(t)$ to denote the number of rewards observed at discrete time stamp $t$ and $\tau_{k}(t)$ to
denote the number of assigned but unobserved rewards for arm $k$ at round $t$,
respectively. The proposed AP-EVT algorithm follows a similar framework to
Algorithm~\ref{alg:GF}. The key difference from EVT lies on how to choose
$S_k(t)$. The $S_k(t)$ in AP-EVT also depends on $\tau_k(t)$ besides $T_k(t)$
and $\hat{\sigma}_{k,T_k(t)}$:
\begin{equation}
  \label{eq:S_k-2}
  (\text{AP-EVT})\quad S_k(t)=\bigg({a\over T_{k}(t) + \delta \tau_k(t)}
  + \sqrt{a\over T_{k}(t) + \delta \tau_k(t)} \hat{\sigma}_{k,T_k(t)}\bigg)^{-1}.
\end{equation}
where again $a$ is a predefined parameter. The coefficient $\delta \in [0,1]$
adjusts the weight of assigned but unobserved rewards.

\begin{algorithm}[H]
  \begin{algorithmic}[1]
    \REQUIRE $b$, $n$, $a$, and {\color{blue} $\delta$} \ENSURE
    $\hat{U}_b(n)=\{k\in [K]~|~\hat{\mu}_{k, T_k(n)} \geq b\}$ \STATE Pull all
    $K$ arms twice and observed rewards $\{X_{k,1}\}_{k=1}^K$ and
    $\{X_{k,2}\}_{k=1}^K$
    \FOR{$t=2K:n$} \STATE Calculate the number of observed rewards $T_k(t)$ and
    {\color{blue}the number of assigned but unobserved rewards $\tau_k(t)$}
    at the current time point $t$ for all arms $k=1,2,\cdots, K$ \STATE Select
    arm
    \[
      I_{t+1} =\argmin_k \quad |\hat{\mu}_{k,T_k(t)}-b| S_k(t)
    \]
    \STATE Send out the order to pull arm $I_{t+1}$ (\textcolor{blue}{but may
      not observe the reward immediately}) 
\ENDFOR
\end{algorithmic}
 \caption{AP-EVT} \label{alg:AP-EVT}
\end{algorithm}

When all rewards are observed, i.e, $\tau_k(t)=0$ for all $k\in [K]$ and $t\le
n$, \eqref{eq:S_k-2} reduces to \eqref{eq:S_k}. Therefore, EVT is a special case
of AP-EVT. In general, arms with unobserved rewards will be pulled less in
AP-EVT than in EVT.

\begin{theorem} \label{thm:AP-EVT} 
Let $K>0, n> 2K$, and consider a thresholding bandit problem. Assume that all
arms of the problem are bounded in $[0,1]$. 
Let $\delta \in [0,1]$ and $\eta\ge 0$. Assume that $\tau_i(t) \le \eta T_i(t)$ for all
arm $i\in [K]$ and $t\le n$, and that $\sum_{i=1}^K \tau_i(t)\le \tau$ for all
$t\le n$.  Let $a = (n-(1-\delta)\tau)/H_{\text{AP-EVT}}$ in \eqref{eq:S_k-2}
with
\[
  H_{\text{AP-EVT}} : = (1+\delta\eta)^2 \sum_{i=1}^K (\sigma_i^2\Delta_i^{-2} + \Delta_i^{-1}).
\]
Given $\epsilon\in (0,1)$, if
\begin{equation}
  n \geq \Theta [H_{\text{AP-EVT}} (\log (nK)  + \log (1/\epsilon))  + (1-\delta)\tau],
  \label{eq:APEVT_complexity}
\end{equation}
where $\Theta$ is some positive constant, then the Algorithm AP-EVT guarantees that with probability at least $1-\epsilon$, the
player can correctly discriminate arms in $U_b$ from those in $U_b^C$.   
\end{theorem}

Note that the parameter $\eta$ here can be considerded as the (maximal) ratio between the
number of unobserved rewards and the number of observed ones.

If all rewards are observed immediately, i.e., $\tau=0$, and thus $\eta =0$, Theorem \ref{thm:AP-EVT} guarantees that the round complexity in the algorithm AP-EVT is consistent to the algorithm EVT, and confirms that the former is a generalization of the latter. {\rcc Note that $\tau$ is usually proportional to the total number of agents in the multi-agent scenario.}

If we set $\delta=0$ in Algorithm \ref{alg:AP-EVT}, and at the same time $\tau$
is dominated by the term $H_{\text{AP-EVT}} (\log (nK)  + \log (1/\epsilon))$,
the round complexity in \eqref{eq:APEVT_complexity} suggests the speedup in
parallelization. In general, we can choose an optimal $\delta$ to minimize the
round complexity in \eqref{eq:APEVT_complexity}.


However, in general we should choose $\delta$ carefully to minimize the round
complexity in \eqref{eq:APEVT_complexity}:
\begin{itemize}
\item When $\tau$ is very large compared to $\eta$, then we choose $\delta$ close to $1$;
\item When  $\tau$ is very small, and thus $\eta$ is small
  too, then we choose $\delta$ close to $0$;
\item When both $\tau$ and $\eta$ are large, we are able to find an
  optimal $\delta$ to minimize the total rounds $n$. 
\end{itemize}

\subsection{Parameter Free Asynchronous Parallel Empirical Variance Guided Thresholding (AP-EVT$_{pf}$)}\label{sec:pf}
In this section, we propose a \emph{parameter free} version of AP-EVT algorithm, which
is anytime and does not require pre-knowledge of total number of rounds or the problem difficulty constant $H$. The key difference from previous algorithms, especially AP-EVT, is that in Step 4 we
use the following quantity to decide which arm to pull
\begin{align*}
  (\text{AP-EVT$_{pf}$})\quad
  B_k(t)  = \sqrt{T_k(t)+\delta \tau_k(t)}
  \left(\sqrt{\hat{\sigma}^2_{k,T_k(t)}+ \hat{\Delta}_{k,T_k(t)}} -
    \hat{\sigma}_{k,T_k(t)}\right).
\end{align*}
Note that $B_k(t)$ no longer depends on $n$, $K$, $H$, or any other parameters.

If every single reward can be observed immediately after a pull, i.e., $\tau=0$, then we
simply set $\delta=0$, and thus AP-EVT$_{pf}$ can lead to a parameter free version of EVT algorithm EVT$_{pf}$. Specifically, in Step 4 of Algorithm \ref{alg:AP-EVT$_{pf}$}, we have
\begin{align*}
  (\text{EVT$_{pf}$})\quad
  B_k(t)  = \sqrt{T_k(t)}
  \left(\sqrt{\hat{\sigma}^2_{k,T_k(t)}+ \hat{\Delta}_{k,T_k(t)}} -
    \hat{\sigma}_{k,T_k(t)}\right).
\end{align*}

In addition to the ``parameter-free'' feature, the AP-EVT$_{pf}$ algorithm also
enjoys similar advantages of AP-EVT over ATP, since the problem complexity constants
$H_{\text{AP-EVT}_{pf}}$ and $H_{\text{AP-EVT}}$ are similar.
This algorithm generates the following result.
\begin{theorem} \label{thm:AP-EVT-pf} 
Let $K>0, n> 2K$, and consider a thresholding bandit problem. Assume that all
arms of the problem are bounded in $[0,1]$. 
Let $\delta\in [0,1]$ and $\eta\ge 0$. Assume that $\tau_i(t) \le \eta T_i(t)$ for all arm $i\in [K]$ and $t\le n$, and
that $\sum_{i=1}^K \tau_i(t)\le \tau$ for all $t\le n$. Given $\epsilon\in (0,1)$, if
\begin{equation}
  \label{eq:roundComplexity-AP-EVTpf}
  n \geq \Theta[H_{\text{AP-EVT}_{pf}} (\log (nK)  + \log (1/\epsilon))  + (1-\delta)\tau] ,
\end{equation}
where
\[
  H_{\text{AP-EVT}_{pf}} : = (1+\delta\eta)\sum_{i=1}^K (\sigma_i^2\Delta_i^{-2} + \Delta_i^{-1}),
\]
and $\Theta$ is some positive constant, then the Algorithm AP-EVT$_{pf}$ guarantees that with probability at least $1-\epsilon$, the
player can correctly discriminate arms in $U_b$ from those in $U_b^C$.   
\end{theorem}

Since EVT$_{pf}$ can be considered as a special case of AP-EVT$_{pf}$ by setting
$\tau=\delta=\eta=0$, we obtain the theoretical guarantee on the round complexity with
   $H_{\text{EVT}_{pf}} : = \sum_{i=1}^K (\sigma_i^2\Delta_i^{-2} +
   \Delta_i^{-1})$,
which is the same as the EVT algorithm.

\begin{algorithm}[H] 
\begin{algorithmic}[1]
\REQUIRE $b$, $n$, and {\color{blue} $\delta$}
\ENSURE $\hat{U}_b(n)=\{k\in [K]~|~\hat{\mu}_{k, T_k(n)} \geq b\}$
\STATE Pull all $K$ arms twice and observed rewards $\{X_{k,1}\}_{k=1}^K$ and $\{X_{k,2}\}_{k=1}^K$
\FOR{$t=2K:n$}
\STATE Calculate the number of observed rewards $T_k(t)$ and \textcolor{blue}{the number of assigned but unobserved rewards $\tau_k(t)$} at the current time point $t$ for all arms $k=1,2,\cdots, K$ 
\STATE Select arm 
\[
     I_{t+1} =\argmin_k B_k(t)
\]
\STATE Send out the order to pull arm $I_{t+1}$ (\textcolor{blue}{but may not observe the reward immediately}) 
\ENDFOR
\end{algorithmic}
 \caption{AP-EVT$_{pf}$} \label{alg:AP-EVT$_{pf}$}
\end{algorithm}

It is worth noting that in the next section, our experiments show that AP-EVT
(or EVT) slightly outperforms AP-EVT$_{pf}$ (or EVT$_{pf}$), due to the
accessibility to additional information such as the problem complexity $H$.

\section{Experiments}

We conduct empirical study to validate the proposed algorithms \APTEV~and
\APTEVf. Section~\ref{sec:exp:1} compares \TEV~(the non-parallel version of
\APTEV) and \TEVf~(the non-parallel version of \APTEVf) to the ATP algorithm
\citep{locatelli2016optimal}. Section~\ref{sec:exp:2} validates the speedup
property of \APTEV~and \APTEVf~or equivalently the tolerance to the number of
unobserved rewards.

\subsection{Comparison among ATP, \TEV, AND \TEVf} \label{sec:exp:1}

All experiments are conducted on synthetic data. Let the total number of arms $K$ be $100$. Arm $k$ follows the uniform distribution $U\left(\mu_k -r_k,\mu_k+r_k\right)$ where $\mu_k$ is generated from the uniform distribution $U\left(0.6, 0.8\right)$. The threshold $b$ is chosen to be 0.7. The parameter $a$ in \TEV~ and \APTEV~ is chosen to be $n / K$ where $n$ is the total number of pulls. 

We compare three approaches ATP, \TEV, and \TEVf~with three settings for $r_k$
which indicates the magnitude of variance to be small: $r_k\sim U\left( 0.15,
  0.25\right)$, median: $r_k\sim U\left( 0.25, 0.35\right)$, and large:$r_k\sim U\left( 0.35, 0.45\right)$.

All random experiments are repeated for 100 times. Figure~\ref{fig:APT_TEV_TEVpf} compares the identification accuracy for all algorithms, that is, the percentage of successful identification of all arms. Note that in \TEV~\emph{for each fixed value of $n$}, the parameter $a$ is chosen by $n/K$ and all experiments are repeated 100 times. We can observe that
\begin{itemize}
\item The proposed two approaches \TEV~and \TEVf~outperform ATP overall, due to estimating the variance on fly;
\item TEV overall outperforms TEV$_{pf}$;
\item The advantage of \TEVf~over APT is more obvious when the variance is small, which is consistent with our theoretical analysis. 
\end{itemize}

\begin{figure*}[htp!]
\centering
\includegraphics[width=0.32\textwidth]{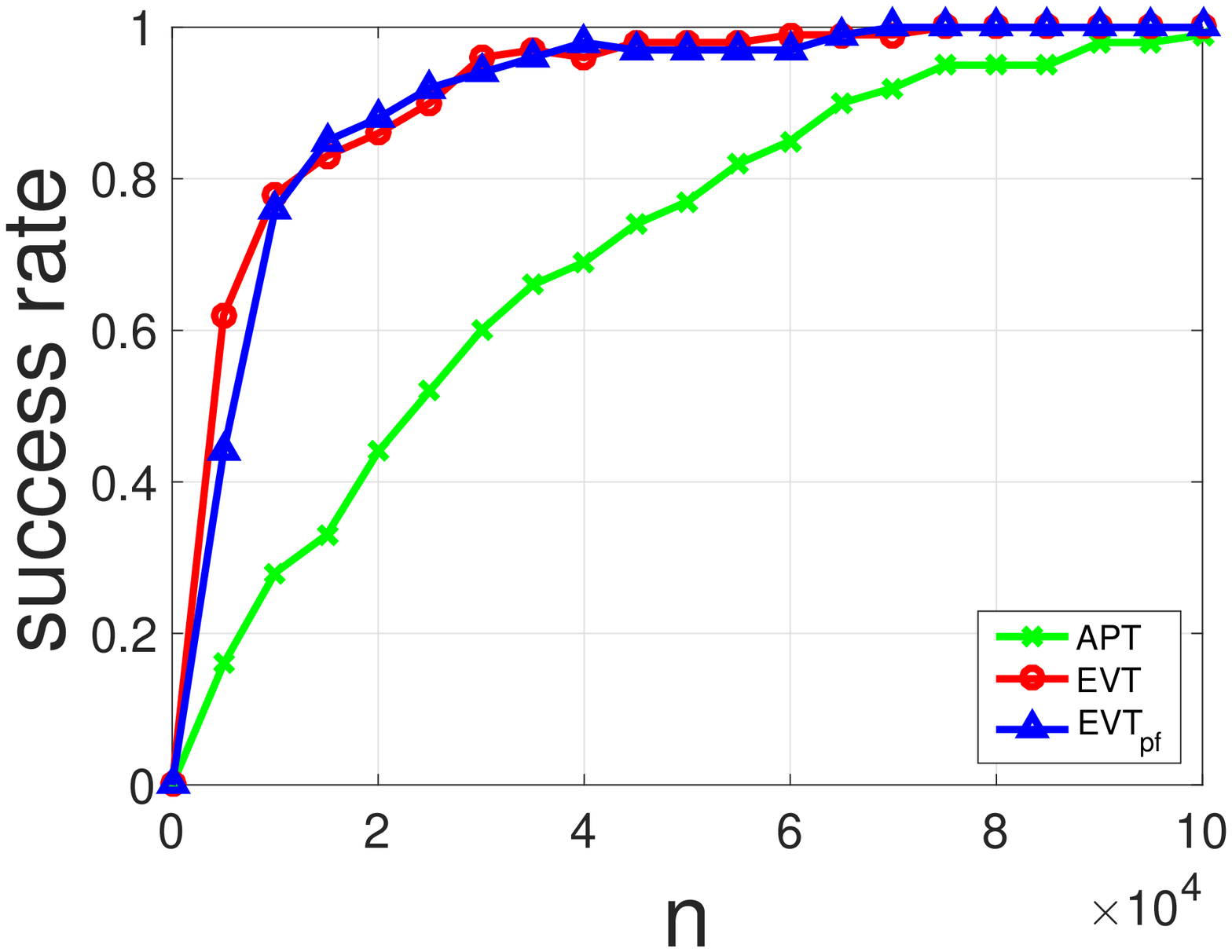} \hspace{-2mm}
\includegraphics[width=0.32\textwidth]{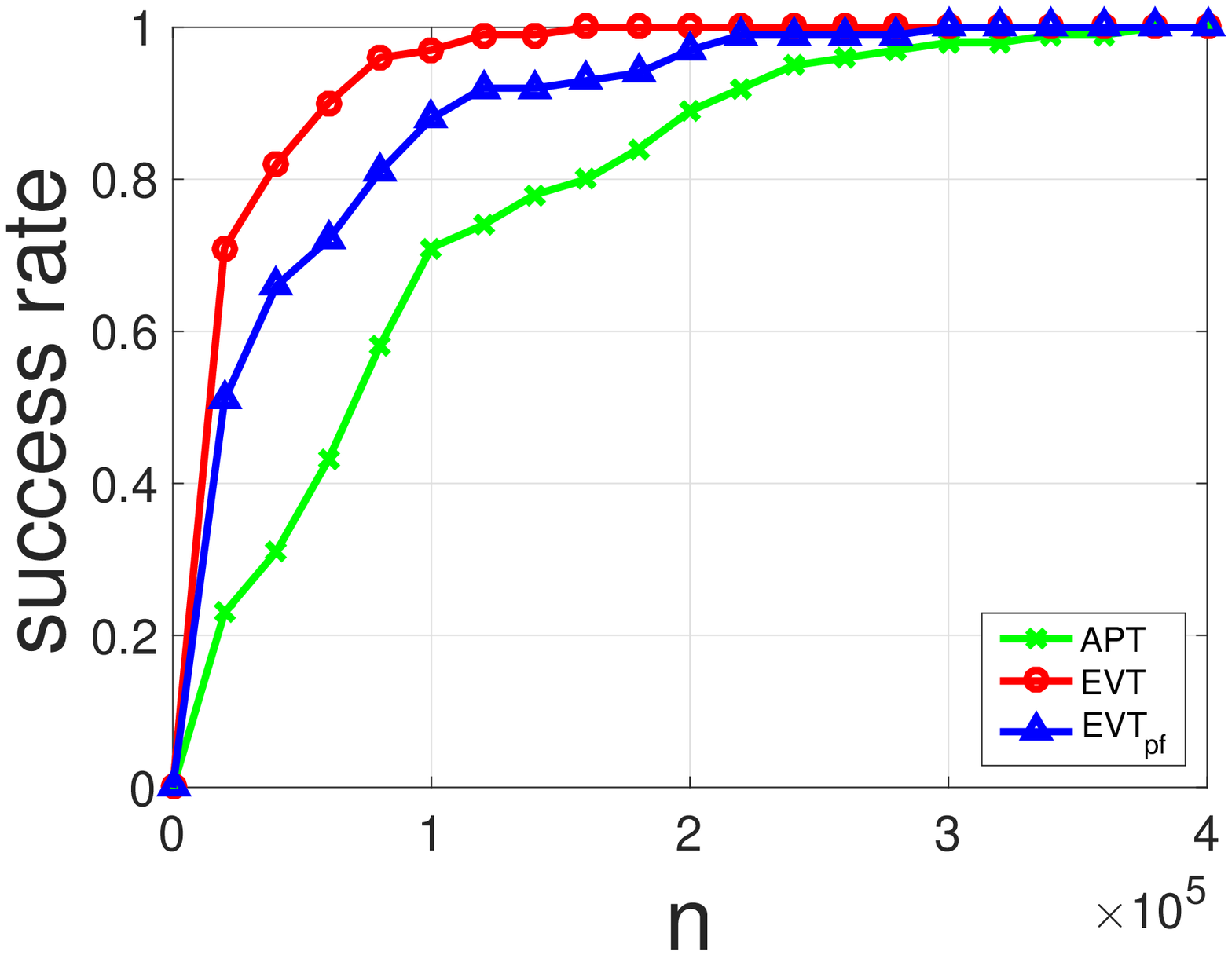} \hspace{-2mm}
\includegraphics[width=0.32\textwidth]{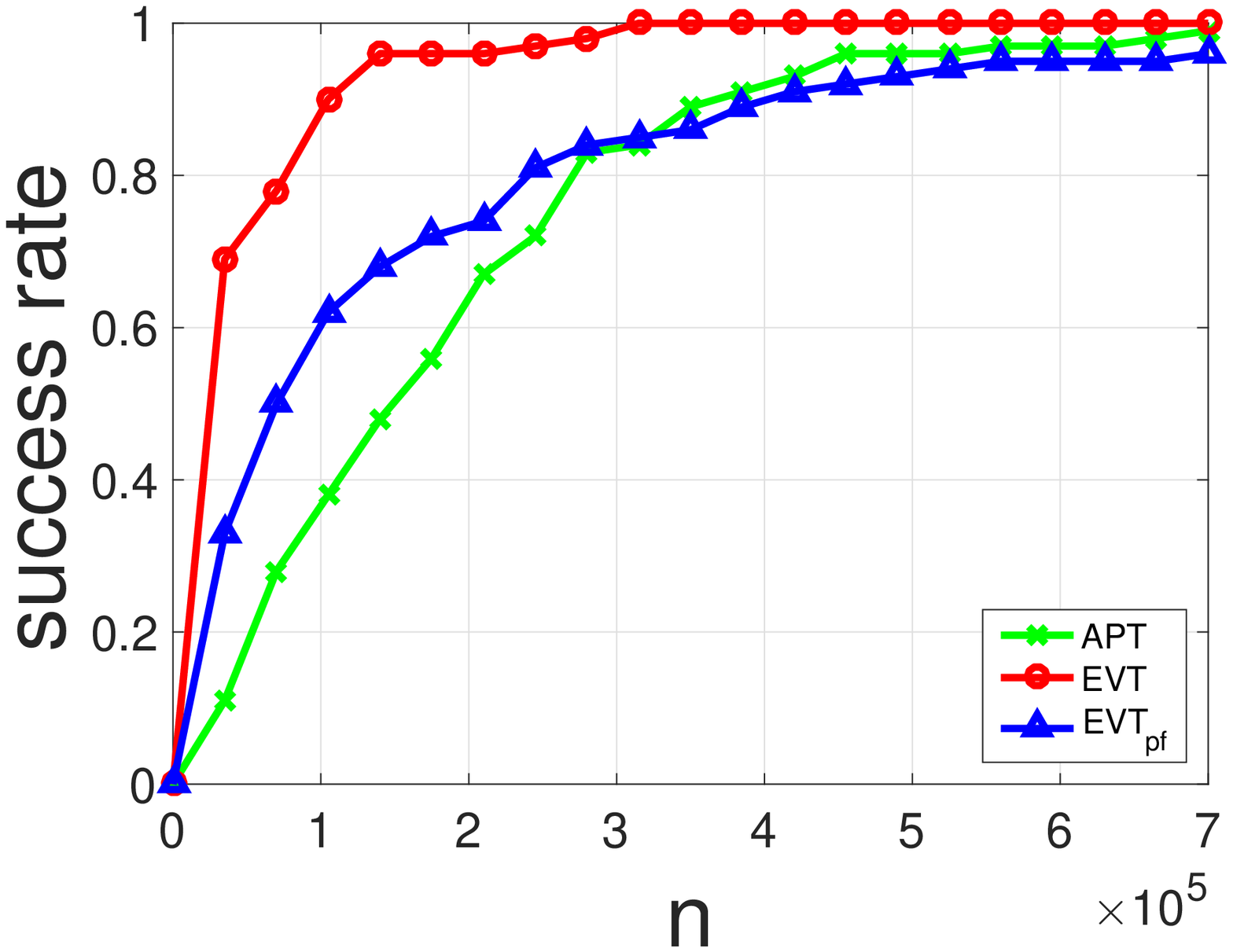}
\caption{Comparison of success rate among APT vs \TEV~vs \TEVf. The variances of three graphs are chosen to be small, median, and large.}
\label{fig:APT_TEV_TEVpf}
\end{figure*} 

\begin{figure*}[htp!]
\centering
{\includegraphics[width=0.32\textwidth]{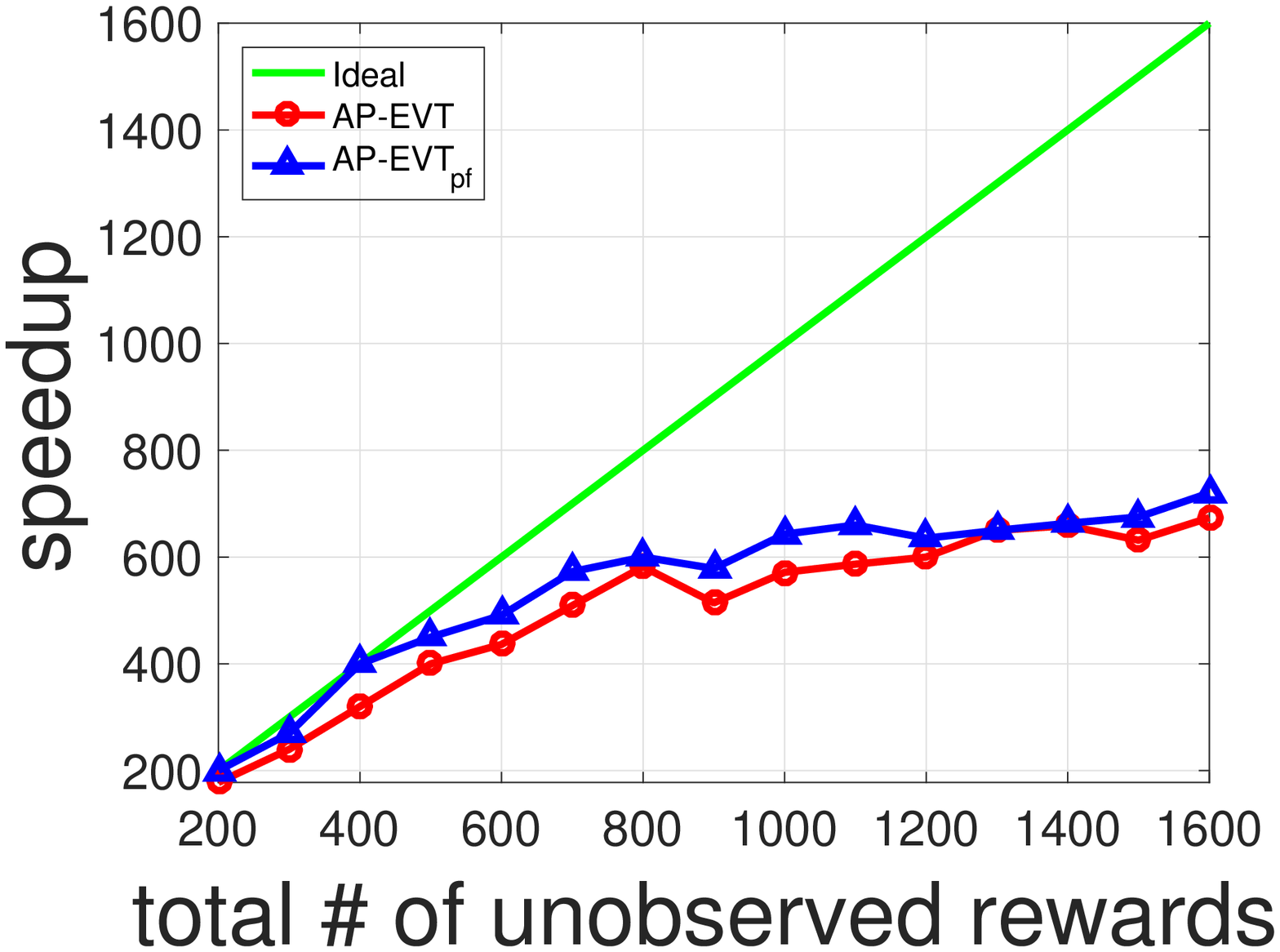}} \hspace{-2mm}
{\includegraphics[width=0.32\textwidth]{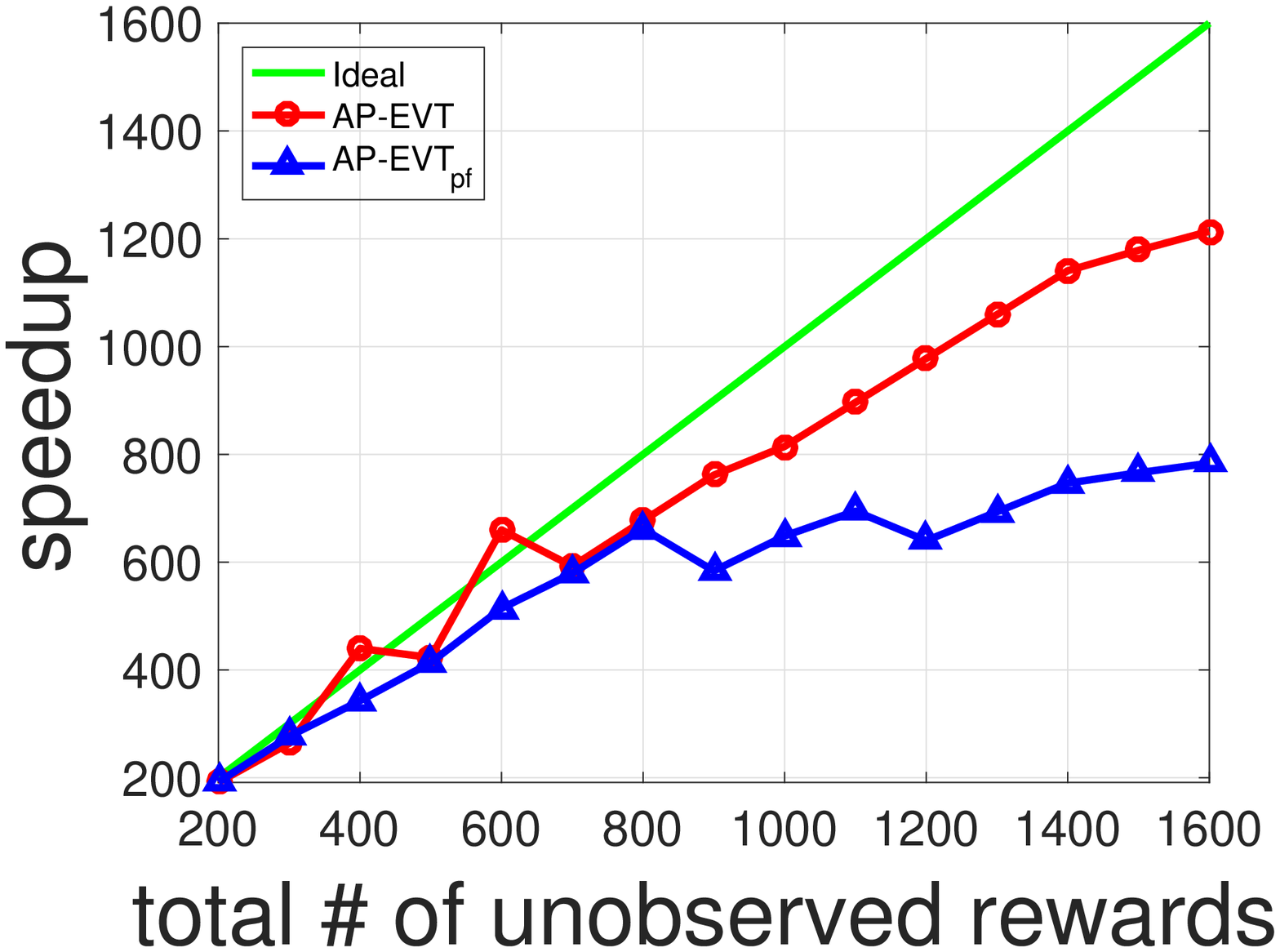}} \hspace{-2mm}
{\includegraphics[width=0.32\textwidth]{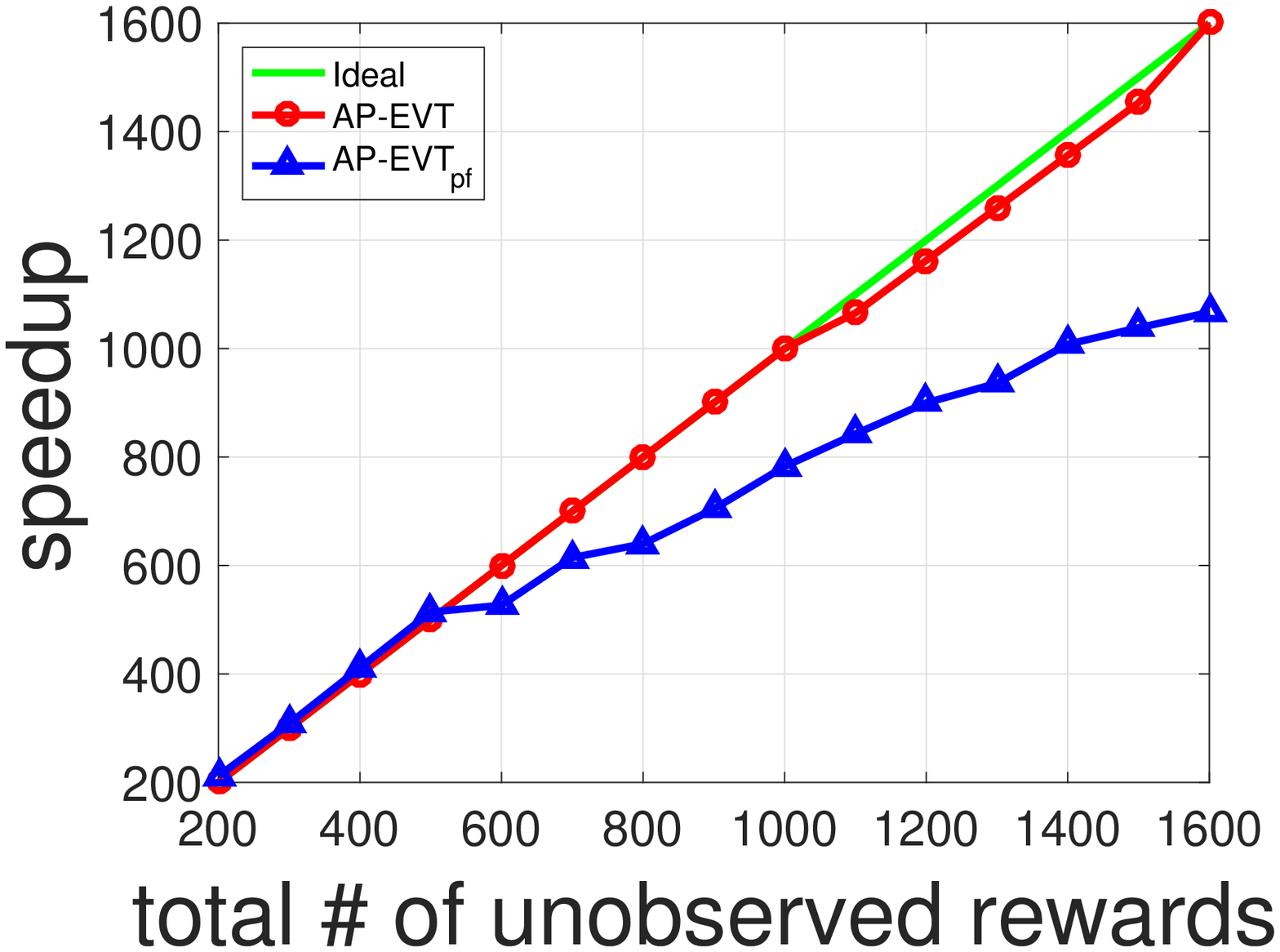}}
\caption{Speedup of \APTEV~and \APTEVf~with small, median and large variances.}
\label{fig:AP_TEV}
\end{figure*} 


\subsection{Speedup of  \APTEV~and \APTEVf} \label{sec:exp:2}

This section validates the speedup of \APTEV~and \APTEVf~or equivalently the tolerance to the number of unobserved rewards. We choose $n=20,000$ and $K=100$. Arm $k$ follows the uniform distribution $U\left(\mu_k -r_k,\mu_k+r_k\right)$, where $\mu_k \sim U\left(0.3, 0.7\right)$. The threshold $b$ is chosen to be 0.5. All random experiments are repeated for 100 times.

We choose a large range of $\tau$ and report the speedup with respect to different number of unobserved rewards $\tau$ in Firgure~\ref{fig:AP_TEV}. As we mentioned before, the value of $\tau$ is usually proportional to the total number of agents. So the speedup with respect to $\tau$ is defined as
\[
  \frac{
      \text{\# of pulls with \emph{full observations}}
      \text{to achieve accuracy 95\%}
  }
  {
      \text{\# of pulls with \emph{maximal $\tau$ unobserved rewards}}
      \text{to achieve accuracy 95\%}
  } \times \tau,
\] 
which measures how many times faster using simultaneously $\tau$ agents than using a single agent.

We compare the speedup of \APTEV~and \APTEVf ~with three settings for $r_k$
which indicate the magnitude of variance from small: $r_k\sim U\left( 0.05,
  0.15\right)$, median: $r_k\sim U\left( 0.1, 0.2\right)$, and large: $r_k\sim U\left( 0.15, 0.25\right)$.

Figure~\ref{fig:AP_TEV} shows the speedup curves for both \APTEV~and \APTEVf. The parameter $\delta$ in \APTEVf~ is simply chosen to be $0$. Key observations include
\begin{itemize}
\item Both algorithms achieve a nice speedup property;
\item \APTEV~overall overperforms \APTEVf;
\item The speedup is better when the variance is large, which is consistent with
  our theory, e.g., in \eqref{eq:APEVT_complexity} if the variance is large, the first term $H_{\text{AP-EVT}} (\log (nK)  + \log (1/\epsilon))$ will dominant the second term $(1-\delta)\tau$ about the maximal number of unobserved rewards.
\end{itemize}

\section{Conclusion and Future Work}

This paper proposes two empirical variance guided parallel algorithms -- AP-EVT
and AP-EVT$_{pf}$ -- for the thresholding bandit problem. Both algorithms
improve the ``optimal'' algorithm \citep{locatelli2016optimal}. First both
algorithms can significantly reduce the \emph{round complexity} if all arms'
variances are small. Second both algorithms support the (asynchronous or
synchronous) parallelization implementation by allowing deciding next pull with
unobserved rewards, which can significantly improve the efficiency in practice. 

Both AP-EVT and AP-EVT$_{pf}$ have similar theoretical guarantees including the
round complexity and the tolerance to the number of unobserved rewards. AP-EVT works
slightly better in practice but needs to know the total number of pulls and the
problem complexity $H$ beforehand and only guarantees the performance after
finishes all pulls, while  AP-EVT$_{pf}$ is totally parameter-free and have the theoretical guarantees along the whole pull path. 

The future work includes studying the lower bound of the round complexity if all arms have bounded high order moments and extending the asynchronous parallelism to other bandit algorithms such as UCB (the leading algorithm for best arm identification).

\section{Acknowledgment}
We thank Dr. Yifei Ma for his constructive comments and helpful advices.

\newpage

\begin{center}
{\Large \bf Supplementary Material}
\end{center}
\setcounter{section}{0}
In this supplement, we will first prove our results: Theorem \ref{thm:EVT},
Theorem \ref{thm:AP-EVT} and Theorem \ref{thm:AP-EVT-pf}. The general cases are
summarized in Proposition \ref{prop:general} and Proposition
\ref{prop:general-2}. Then we prove Theorem \ref{thm:lowerbound} on the lower
bound of probability of making a mistake.

\paragraph{Parameter-dependent case: Theorem \ref{thm:EVT} and Theorem \ref{thm:AP-EVT}}
We first prove the case where the algorithm involves the complexity parameter
$H$, that is,
\[
  I_{t+1} = \argmin_{i\in [K]} B_i(t),
\]
where
\[
  B_i(t) = \hat{\Delta}_i(t) \left(\sqrt{\frac{2\hat{\sigma}_{i,T_i(t)}^2a}{T_i(t) + \delta \tau_i(t)}} + \frac{3a}{T_i(t) +\delta\tau_i(t)}\right)^{-1},
\]
and
\[
  a \le \frac{n-K-(1-\delta)\tau}{H}.
\]

The idea of proving our main results is based on the so-called {\it optimism in
  face of uncertainty} \citep{bubeck2012regret}.

First, we find an event, which is
``consistent'' with the data, in the sense that the more data is sampled, the
bigger probability of the event is. This step is often through certain
concentration inequalities. 

\begin{lemma}
  Assume $X_{i,s}$ takes values in $[0,1]$. For $a>0$, define
  \begin{equation}
  \label{eq:A1}
  A_1 = \left\{\forall i\in [K], \forall t\le n : |\hat{\mu}_{i,t} - \mu_i| \le \sqrt{\frac{2\hat{\sigma}_{i,t}^2 a}{t}} + \frac{(3-\sqrt{2})a}{t}\right\},
\end{equation}
and
\begin{equation}
  \label{eq:A2}
  A_2 = \left\{\forall i\in [K], \forall t\le n : |\hat{\sigma}_{i,t} - \sigma_i| \le \sqrt{\frac{a}{4t}}\right\}.
\end{equation}
Let $A: = A_1\cap A_2$, then
\begin{equation}
  \label{eq:concentration}
  \bP(A)  \ge 1 - 5 nK \exp(-a/8).
\end{equation}
\end{lemma}
\begin{proof}
  Define
  \[
    B = \left\{\forall i\in [K], \forall t\le n : |\hat{\mu}_{i,t} - \mu_i| \le \sqrt{\frac{\hat{\sigma}_{i,t}^2 a}{t}} + \frac{3a}{2t}\right\}.
  \]
  Applying Theorem $1$ in \citet{audibert2009exploration} yields
  \[
    \bP(B) \ge 1 - 3nK \exp(-a/2).
  \]
  It is obvious that $B \subseteq A_1$, so
  \begin{equation}
    \label{eq:pA1}
    \bP(A_1) \ge 1 - 3nK \exp(-a/2).
  \end{equation}
  Applying Theorem $10$ in \citet{maurer2009empirical} yields
  \begin{equation}
    \label{eq:pA2}
    \bP(A_2) \ge 1 - 2 nK \exp(-a/8).
  \end{equation}
  It follows from \eqref{eq:pA1} and \eqref{eq:pA2} that
  \begin{align*}
    \bP(A_1\cap A_2)
    & = \bP(A_1) + \bP(A_2) - \bP(A_1\cup A_2)\\
    & \ge \bP(A_1) + \bP(A_2) - 1\\
    & \ge 1 - 3 nK \exp(-a/2) - 2 nK\exp(-a/8)\\
    & \ge 1 - 5 nK \exp(-a/8),
  \end{align*}
  which completes the proof.
\end{proof}

Next, the most ``favorable'' environment should be identified on the event
found previously. In our case, we should be able to correctly discriminate arms
above the threshold from those below the threshold on the event $A$.

There is a simple method to do the correct separation: for each arm $i\in [K]$,
we require at the final round $n$,
\[
  \begin{cases}
    \hat{\mu}_{i,T_i(n)} \ge \mu_i - \frac{\mu_i-b}{2},\quad\mbox{if}~\mu_i\ge b;\\
    \hat{\mu}_{i,T_i(n)} \le \mu_i + \frac{b-\mu_i}{2},\quad\mbox{if}~\mu_i< b.
  \end{cases}
\]
In other words,
\[
  |\hat{\mu}_{i,T_i(n)}-\mu_i| \le \frac{\Delta_i}{2},
\]
where $\Delta_i = |\mu_i-b|$.

Since we are on event $A$, it is enough to show that
\begin{equation}
  \label{eq:main-inq}
  \sqrt{\frac{2\hat{\sigma}_{i,T_i(n)}^2 a}{T_i(n)}} + \frac{(3-\sqrt{2})a}{T_i(n)} \le \frac{\Delta_i}{2},
\end{equation}
which is a sufficient condition to correctly discriminate arms.

Define $\cL(n)$ to be the event that at least one arm is incorrectly discriminated, that is, 
\begin{equation}
  \label{eq:L}
  \cL(n) = \mathbbm{1}(\{U_b\cap \hat{U}_b^C(n) \neq \varnothing\}\cup \{U^C_b \cap
  \hat{U}_b(n) \neq \varnothing\}), 
\end{equation}
where $\hat{U}_b(n)$ is the output of an algorithm after pulling $n$ rounds of
arms, and $\hat{U}_b^C(n)$ is its complement.
Then the expected value of
$\cL(n)$ is exactly the probability of making a
mistake:
\[
  \E(\cL(n)) = \bP(\{U_b\cap \hat{U}_b^C(n) \neq \varnothing\}\cup \{U^C_b \cap
  \hat{U}_b(n) \neq \varnothing\}).
\]

As long as \eqref{eq:main-inq} holds on event $A$, we obtain following result:
\begin{proposition}\label{prop:general}
 Let $K, n>0$, $\eta\ge 0, \delta\in [0,1]$ and $b\in \R$. Assume that the distribution for the
 outcome of each arm is bounded in
 $[0,1]$. Assume that $\tau_i(t) \le \eta T_i(t)$ for all arm $i\in [K]$ and
 $t\le n$, and that $\sum_{i=1}^K \tau_i(t)\le \tau$ for all $t\le n$.

 Algorithm $1$'s expected loss in the asynchronous thresholding multi-bandit
 problem is upper bounded by
 \begin{equation}
   \label{eq:main}
   \E (\cL(n)) \le 5nK \exp(-a/8),
 \end{equation}
 where
 \[
   a \le \frac{n-K-(1-\delta)\tau}{H},
 \]
\end{proposition}
with
\begin{equation}
  \label{eq:H}
  H : = \sum_{i=1}^K  h_i, \quad h_i = C_1\sigma_i^2\Delta_i^{-2} + C_2\Delta_i^{-1},
\end{equation}
and
\begin{equation}
  \label{eq:C1C2}
  C_1 = 32(1+\delta\eta)^2,~\text{and}~ C_2 = 8\sqrt{2}(1+\delta\eta)^{3/2} + 48(1+\delta\eta).
\end{equation}
\begin{proof}
  This is because
  \begin{align*}
    \bP(\{U_b\cap \hat{U}_b^C =\varnothing\}\cap \{U_b^C\cap \hat{U}_b = \varnothing\})
    & = 1 - \bP(\{U_b\cap \hat{U}_b^C \neq \varnothing\}\cup \{U_b^C\cap \hat{U}_b \neq \varnothing\})\\
    & = 1- \E(\cL(n))\\
    & \ge 1- 5nK \exp(-a/8).
  \end{align*}
\end{proof}

Next, we focus on proving \eqref{eq:main-inq}. To this end, we need a
characterization of some helpful arm.
\begin{lemma}
  \label{lem:arm_k}
  Under the same assumptions as Proposition \ref{prop:general}, there exists an arm $k\in [K]$ such that
  \begin{equation}
    \label{eq:arm_k}
    \Delta_k \left(\sqrt{\frac{2\hat{\sigma}_{k,T_k(t)}^2a}{T_k(t) + \delta
          \tau_k(t)}} + \frac{3a}{T_k(t) +\delta\tau_k(t)}\right)^{-1} \ge 4(1 + \delta\eta),
  \end{equation}
  where $\tilde{t} = t + 1$ is the last round when arm $k$ is pulled.
\end{lemma}
\begin{proof}
  At the final round $n$, consider an arm $k$ such that
  \begin{equation}
    \label{eq:Cauchy}
    T_k(n) + \tau_k(n) -1 - (1-\delta)\tau_k(t) \ge \frac{(n-K-(1-\delta)\tau)h_k}{H},
  \end{equation}
  where $H$ and $h_k$ are defined in \eqref{eq:H}. We know that such an arm
  exists, otherwise we have
  \begin{align*}
    n-K-(1-\delta)\tau
    & > \sum_{i=1}^K[T_i(n) + \tau_i(n) -1 - (1-\delta)\tau_i(t)]\\
    & \ge n- K - (1-\delta)\tau,
  \end{align*}
  which is a contradiction.

  By the definition of $t$ (via $\tilde{t}$) and \eqref{eq:Cauchy}, we have
  \begin{align*}
    T_k(t) + \delta\tau_k(t)
    & = T_k(t) + \tau_k(t) - (1-\delta)\tau_k(t)\\
    & = T_k(\tilde{t}) + \tau_k(\tilde{t}) - 1 -(1-\delta)\tau_k(t)\\
    & = T_k(n) + \tau_k(n) - 1 -(1-\delta)\tau_k(t)\\
    & \ge \frac{(n-K-(1-\delta)\tau)h_k}{H}.
  \end{align*}
  Set 
  \[
    s_k : = \frac{T_k(t)+\delta \tau_k(t)}{a}.
  \]
  Since
  \[
    a \le \frac{n-K-(1-\delta)\tau}{H},
  \]
  we have
  \begin{equation}
    \label{eq:s-h}
    s_k \ge h_k.
  \end{equation}
  Noting that
  \begin{align}
    \sqrt{h_k} & = \frac{\sqrt{(4\sqrt{2}(1+\delta\eta)\sigma_k)^2 + (8\sqrt{2}(1+\delta\eta)^{3/2} +
          48(1+\delta\eta))\Delta_k}}{\Delta_k}  \notag\\
    & \ge \frac{4\sqrt{2}(1+\delta\eta)\sigma_k +
      \sqrt{(4\sqrt{2}(1+\delta\eta)\sigma_k)^2 + (8\sqrt{2}(1+\delta\eta)^{3/2} +
      48(1+\delta\eta))\Delta_k}}{2\Delta_k}\notag\\
               & = : x_k,\label{eq:xk}
  \end{align}
  where $x_k$ is the (only) positive solution of the following quadratic equation
  \begin{equation}
    \label{eq:quadratic}
    \Delta_k x_k^2 - 4\sqrt{2}(1+\delta\eta)\sigma_k x_k -
    2\sqrt{2}(1+\delta\eta)^{3/2} - 12(1+\delta\eta) = 0.
  \end{equation}
  Then it follows from \eqref{eq:s-h}, \eqref{eq:xk} and \eqref{eq:quadratic},
  we get
  \[
    \Delta_k s_k - 4\sqrt{2}(1+\delta\eta)\sigma_k \sqrt{s}_k - 
    2\sqrt{2}(1+\delta\eta)^{3/2} - 12(1+\delta\eta) \ge 0.
  \]
  Since on event $A_2$, we have
  \begin{align*}
    \hat{\sigma}_{k,T_k(t)} & \le \sigma_k + \sqrt{\frac{a}{4T_k(t)}} = \sigma_k +
    \sqrt{\frac{1}{4s_k}}\sqrt{\frac{T_k(t)+\delta \tau_k}{T_k(t)}}\\
                            & \le \sigma_k + \sqrt{\frac{(1+\delta\eta)}{4s_k}},
  \end{align*}
  and thus
    \[
      \Delta_k s_k - 4\sqrt{2}(1+\delta\eta)\hat{\sigma}_{k,T_k(t)}\sqrt{s_k} - 12(1+\delta\eta) \ge 0,
    \]
    which is equivalent to the desired inequality \eqref{eq:arm_k} and completes
    the proof.
\end{proof}

Now it's ready to prove the following lemma:
\begin{lemma}
  \label{lem:main-inq}
  Under the same assumptions as Proposition \ref{prop:general}, the inequality
  \eqref{eq:main-inq} holds on event $A := A_1\cap A_2$, where $A_1$ and $A_2$
  are defined in \eqref{eq:A1} and
  \eqref{eq:A2}, respectively.
\end{lemma}
\begin{proof}
  Choose arm $k$ satisfying \eqref{eq:arm_k} in Lemma \ref{lem:arm_k}.

  Since $I_{t+1} = k$, we have for each $i\in [K]$,
  \begin{equation}
    \label{eq:11}
    B_k(t) \le B_i(t).
  \end{equation}
  By triangle inequality, we have for every $i\in [K]$,
  \[
    |\hat{\Delta}_{i,T_i(t)} - \Delta_i| \le |\hat{\mu}_{i,T_i(t)}-\mu_i|.
  \]
  Therefore, on event $A_1$, we have
  \[
    |\hat{\Delta}_{i,T_i(t)} - \Delta_i| \le \sqrt{\frac{2\hat{\sigma}_{i,T_i(t)}^2a}{T_i(t)}} + \frac{(3-\sqrt{2})a}{T_i(t)},
  \]
  which automatically implies that
  \begin{equation}
    \label{eq:10}
     -\sqrt{\frac{2\hat{\sigma}_{i,T_i(t)}^2a}{T_i(t)}} - \frac{3a}{T_i(t)} \le \hat{\Delta}_{i,T_i(t)} -\Delta_i\le \sqrt{\frac{2\hat{\sigma}_{i,T_i(t)}^2a}{T_i(t)}} + \frac{3a}{T_i(t)}.
  \end{equation}
  Using the upper bound in \eqref{eq:10}, we obtain
  \begin{align*}
    B_i(t)
    & = (\hat{\Delta}_{i,T_i(t)} - \Delta_i)\left(\sqrt{\frac{2\hat{\sigma}_{i,T_i(t)}^2a}{T_i(t) + \delta \tau_i(t)}} + \frac{3a}{T_i(t) +\delta\tau_i(t)}\right)^{-1}\\
    & \quad + \Delta_i \left(\sqrt{\frac{2\hat{\sigma}_{i,T_i(t)}^2a}{T_i(t) + \delta \tau_i(t)}} + \frac{3a}{T_i(t) +\delta\tau_i(t)}\right)^{-1}\\
    & \le \frac{\sqrt{\frac{2\hat{\sigma}_{i,T_i(t)}^2a}{T_i(t)}} + \frac{3a}{T_i(t)}}{\sqrt{\frac{2\hat{\sigma}_{i,T_i(t)}^2a}{T_i(t) + \delta \tau_i(t)}} + \frac{3a}{T_i(t) +\delta\tau_i(t)}}\\
    & \quad + \Delta_i\left(\sqrt{\frac{2\hat{\sigma}_{i,T_i(t)}^2a}{T_i(t)}} + \frac{3a}{T_i(t)}\right)^{-1}\cdot \frac{\sqrt{\frac{2\hat{\sigma}_{i,T_i(t)}^2a}{T_i(t)}} + \frac{3a}{T_i(t)}}{\sqrt{\frac{2\hat{\sigma}_{i,T_i(t)}^2a}{T_i(t) + \delta \tau_i(t)}} + \frac{3a}{T_i(t) +\delta\tau_i(t)}}.
  \end{align*}
  Since $\tau_i \le \eta T_i(t)$,
  \begin{align}
    \frac{\sqrt{\frac{2\hat{\sigma}_{i,T_i(t)}^2a}{T_i(t)}} + \frac{3a}{T_i(t)}}{\sqrt{\frac{2\hat{\sigma}_{i,T_i(t)}^2a}{T_i(t) + \delta \tau_i(t)}} + \frac{3a}{T_i(t) +\delta\tau_i(t)}}
     & = \frac{\sqrt{2\hat{\sigma}_{i,T_i(t)}^2a T_i(t)} + 3a }{\sqrt{\frac{T_i(t)}{T_i(t) + \delta\tau_i(t)}}\sqrt{2\hat{\sigma}_{i,T_i(t)}^2 a T_i(t)} + \frac{3a T_i(t)}{T_i(t) + \delta \tau_i(t)}} \notag\\
    & \le \frac{\sqrt{2\hat{\sigma}_{i,T_i(t)}^2 a T_i(t)} + 3a }{\sqrt{\frac{1}{1+\delta\eta}}\sqrt{2\hat{\sigma}_{i,T_i(t)}^2 a T_i(t)} + \frac{1}{1+\delta\eta} \cdot 3a}\notag\\
    & \le 1 + \delta\eta,\label{eq:eta-tau}
  \end{align}
  and thus
  \begin{equation}
    \label{eq:13}
    B_i(t) \le (1+\delta\eta)\cdot \Delta_i\left(\sqrt{\frac{2\hat{\sigma}_{i,T_i(t)}^2a}{T_i(t)}} + \frac{3a}{T_i(t)}\right)^{-1} + (1+\delta \eta).
  \end{equation}
  Similarly, using the lower bound in \eqref{eq:10}, we get
  \begin{align*}
    B_k(t)
    & = (\hat{\Delta}_{k,T_k(t)} - \Delta_k)\left(\sqrt{\frac{2\hat{\sigma}_{k,T_k(t)}^2a}{T_k(t) + \delta \tau_k(t)}} + \frac{3a}{T_k(t) +\delta\tau_k(t)}\right)^{-1}\\
    & \quad + \Delta_k \left(\sqrt{\frac{2\hat{\sigma}_{k,T_k(t)}^2a}{T_k(t) + \delta \tau_k(t)}} + \frac{3a}{T_k(t) +\delta\tau_k(t)}\right)^{-1}\\
    & \ge - \frac{\sqrt{\frac{2\hat{\sigma}_{k,T_k(t)}^2a}{T_k(t)}} + \frac{3a}{T_k(t)}}{\sqrt{\frac{2\hat{\sigma}_{k,T_k(t)}^2a}{T_k(t) + \delta \tau_k(t)}} + \frac{3a}{T_k(t) +\delta\tau_k(t)}}\\
    & \quad + \Delta_k \left(\sqrt{\frac{2\hat{\sigma}_{k,T_k(t)}^2a}{T_k(t) + \delta \tau_k(t)}} + \frac{3a}{T_k(t) +\delta\tau_k(t)}\right)^{-1}.
  \end{align*}
  and so
  \begin{equation}
    \label{eq:12}
    B_k(t) \ge  \Delta_k \left(\sqrt{\frac{2\hat{\sigma}_{k,T_k(t)}^2a}{T_k(t) + \delta \tau_k(t)}} + \frac{3a}{T_k(t) +\delta\tau_k(t)}\right)^{-1} -(1+\delta\eta),
  \end{equation}
  by \eqref{eq:eta-tau}.
  Combining \eqref{eq:13} and \eqref{eq:12} yields
  \[
    \Delta_k \left(\sqrt{\frac{2\hat{\sigma}_{k,T_k(t)}^2a}{T_k(t) + \delta
          \tau_k(t)}} + \frac{3a}{T_k(t) +\delta\tau_k(t)}\right)^{-1}
     \le (1+\delta\eta)\cdot \Delta_i\left(\sqrt{\frac{2\hat{\sigma}_{i,T_i(t)}^2a}{T_i(t)}} + \frac{3a}{T_i(t)}\right)^{-1} + 2(1+\delta \eta).
  \]
  Therefore,
  \begin{align*}
    \Delta_i\left(\sqrt{\frac{2\hat{\sigma}_{i,T_i(t)}^2a}{T_i(t)}} +
    \frac{3a}{T_i(t)}\right)^{-1}
    & \ge \frac{\Delta_k \left(\sqrt{\frac{2\hat{\sigma}_{k,T_k(t)}^2a}{T_k(t) + \delta
      \tau_k(t)}} + \frac{3a}{T_k(t) +\delta\tau_k(t)}\right)^{-1} - 2(1+\delta\eta)}{1+\delta\eta}\\
    & \ge 2
  \end{align*}
  by \eqref{eq:arm_k} in Lemma \ref{lem:arm_k}, or equivalently
  \[
    \sqrt{\frac{2\hat{\sigma}_{i,T_i(t)}^2a}{T_i(t)}} +
    \frac{3a}{T_i(t)} \le \frac{\Delta_i}{2}.
  \]
  Since we are on event $A_2$ defined in \eqref{eq:A2}, we get
  \[
    \sqrt{\frac{2a}{T_i(t)}}\left(\sigma_i - \sqrt{\frac{a}{4T_i(t)}}\right) +
    \frac{3a}{T_i(t)} \le \frac{\Delta_i}{2},
  \]
  and thus
  \[
    \sqrt{\frac{2a}{T_i(t)}}\sigma_i  +
    \left(3-\frac{\sqrt{2}}{2}\right) \frac{a}{T_i(t)} \le \frac{\Delta_i}{2}.
  \]
  Using the fact that $T_i(t)\le T_i(n)$, we have
  \[
    \sqrt{\frac{2a}{T_i(n)}}\sigma_i  +
    \left(3-\frac{\sqrt{2}}{2}\right) \frac{a}{T_i(n)} \le \frac{\Delta_i}{2}.
  \]
  Now we use the estimate
  \[
    \hat{\sigma}_{i,T_i(n)} \le \sigma_i + \sqrt{\frac{a}{4T_i(n)}},
  \]
  due to the event $A_2$, and obtain
  \[
    \sqrt{\frac{2\hat{\sigma}_{i,T_i(n)}^2a}{T_i(n)}}+  \frac{(3-\sqrt{2})a}{T_i(n)} \le \frac{\Delta_i}{2},
  \]
  which is exactly \eqref{eq:main-inq}, and completes the proof. 
\end{proof}

\paragraph{Proof of Theorem \ref{thm:EVT}}
\begin{proof}
  Since there is no staleness in this situation, we can simply set $\tau=0$ in
  Proposition \ref{prop:general}, and also $\delta=\eta=0$.

  By the assumption $T>2K$, we have
  \[
    n- K > \frac{n}{2},
  \]
  and thus we can choose
  \[
    a = \frac{n}{2H},
  \]
  in Proposition \ref{prop:general}. Then the probability that makes a mistake
  is at most $5 nK\exp(-n/16H)$ by \eqref{eq:main}. If we require that the
  probability we can correctly discriminate arms is at least $1-\epsilon$, it
  suffices to have
  \[
    5 nK\exp(-n/16H) \le \epsilon,
  \]
  which is equivalent to
  \[
    n \ge \Theta(H_{\text{EVT}}(\log(nK) + \log(1/\epsilon))),
  \]
  for some constant $\Theta>0$.
\end{proof}
\paragraph{Proof of Theorem \ref{thm:AP-EVT}}
\begin{proof}
  This is a direct application of Proposition \ref{prop:general}.

  By the assumption $T>2K$, we have
  \[
    n- K > \frac{n}{2},
  \]
  and thus we can choose
  \[
    a = \frac{n/2-(1-\delta)\tau}{H},
  \]
  in Proposition \ref{prop:general}. Then the probability that makes a mistake
  is at most $5 nK\exp((-n/2 + (1-\delta)\tau)/H)$ by \eqref{eq:main}. If we require that the
  probability we can correctly discriminate arms is at least $1-\epsilon$, it
  suffices to have
  \[
    5 nK\exp((-n/2 + (1-\delta)\tau)/H) \le \epsilon,
  \]
  which is equivalent to
  \[
    n \ge \Theta(H_{\text{AP-EVT}}(\log(nK) + \log(1/\epsilon))+ (1-\delta)\tau),
  \]
  for some constants $\Theta>0$.
\end{proof}

\paragraph{Parameter free case: Theorem \ref{thm:AP-EVT-pf}}

Now let's consider the parameter free case, that is,
\[
  I_{t+1} = \argmin_{i\in [K]} B_i(t),
\]
where
\begin{equation}
  \label{eq:B-tau}
  B_i(t) = \sqrt{T_i(t)+\delta \tau_i(t)}\left(\sqrt{\hat{\sigma}^2_{i,T_i(t)}+ \hat{\Delta}_{i,T_i(t)}} - \hat{\sigma}_{i,T_i(t)}\right).
\end{equation}

The proofs of parameter free case are similar to the previous arguments.

First, we need to introduce the concentration inequalities.
\begin{lemma}
  Assume $X_{i,s}$ takes values in $[0,1]$. For $a>0$, define
  \begin{equation}
  \label{eq:E1}
  E_1 = \left\{\forall i\in [K], \forall t\le n : |\hat{\mu}_{i,t} - \mu_i| \le \sqrt{\frac{\hat{\sigma}_{i,t}^2 a}{t}} + \frac{(2-\sqrt{2})a}{4t}\right\},
\end{equation}
and
\begin{equation}
  \label{eq:E2}
  E_2 = \left\{\forall i\in [K], \forall t\le n : |\hat{\sigma}_{i,t} - \sigma_i| \le \sqrt{\frac{a}{32t}}\right\}.
\end{equation}
Let $E: = E_1\cap E_2$, then
\begin{equation}
  \label{eq:concentration-2}
  \bP(E)  \ge 1 - 5 nK \exp(-a/64).
\end{equation}
\end{lemma}
\begin{proof}
  Define
  \[
    D = \left\{\forall i\in [K], \forall t\le n : |\hat{\mu}_{i,t} - \mu_i| \le \sqrt{\frac{\hat{\sigma}_{i,t}^2 a}{12t}} + \frac{a}{8t}\right\}.
  \]
  Applying Theorem $1$ in \citet{audibert2009exploration} yields
  \[
    \bP(D) \ge 1 - 3nK \exp(-a/24).
  \]
  It is obvious that $D \subseteq E_1$, so
  \begin{equation}
    \label{eq:pE1}
    \bP(E_1) \ge 1 - 3nK \exp(-a/24).
  \end{equation}
  Applying Theorem $10$ in \citet{maurer2009empirical} yields
  \begin{equation}
    \label{eq:pE2}
    \bP(E_2) \ge 1 - 2 nK \exp(-a/8).
  \end{equation}
  It follows from \eqref{eq:pE1} and \eqref{eq:pE2} that
  \begin{align*}
    \bP(E_1\cap E_2)
    & = \bP(E_1) + \bP(E_2) - \bP(E_1\cup E_2)\\
    & \ge \bP(E_1) + \bP(E_2) - 1\\
    & \ge 1 - 3 nK \exp(-a/24) - 2 nK\exp(-a/64)\\
    & \ge 1 - 5 nK \exp(-a/64),
  \end{align*}
  which completes the proof.
\end{proof}

As discussed in the parameter-dependent case, the following is a sufficient condition to correctly discriminate arms:
\begin{equation}
  \label{eq:main-inq-2}
  \sqrt{\frac{\hat{\sigma}_{i,T_i(n)}^2 a}{T_i(n)}} + \frac{(2-\sqrt{2})a}{4T_i(n)} \le \frac{\Delta_i}{2}.
\end{equation}

As long as \eqref{eq:main-inq-2} holds on event $E$, we obtain following result:
\begin{proposition}\label{prop:general-2}
 Let $K, n>0$, $\eta\ge 0, \delta\in [0,1]$ and $b\in \R$. Assume that the distribution for the
 outcome of each arm is bounded in
 $[0,1]$. Assume that $\tau_i(t) \le \eta T_i(t)$ for all arm $i\in [K]$ and
 $t\le n$, and that $\sum_{i=1}^K \tau_i(t)\le \tau$ for all $t\le n$.

 Algorithm AP-EVT$_{pf}$'s expected loss in the asynchronous thresholding multi-bandit
 problem is upper bounded by
 \begin{equation}
   \label{eq:main-2}
   \E (\cL(n)) \le 5nK \exp(-a/64),
 \end{equation}
 where
 \[
   a \le \frac{n-K-(1-\delta)\tau}{16H},
 \]
\end{proposition}
with
\begin{equation}
  \label{eq:H-2}
  H : = \sum_{i=1}^K  h_i, \quad h_i = (1+\delta\eta)(\sigma_i^2\Delta_i^{-2} + \Delta_i^{-1}),
\end{equation}

Next, we focus on proving \eqref{eq:main-inq-2}. To this end, we need a
characterization of some helpful arm.
\begin{lemma}
  \label{lem:arm_k-2}
  Under the same assumptions as Proposition \ref{prop:general-2}, there exists an arm $k\in [K]$ such that
  \begin{equation}
    \label{eq:arm_k-2}
    \sqrt{T_k(t)+\delta\tau_k(t)}\left(\sqrt{\sigma_k^2 + \Delta_k}-\sigma_k\right)\ge 4\sqrt{1+\delta\eta} \sqrt{a},
  \end{equation}
  where $\tilde{t} = t + 1$ is the last round when arm $k$ is pulled.
\end{lemma}
\begin{proof}
  At the final round $n$, consider an arm $k$ such that
  \begin{equation}
    \label{eq:Cauchy-2}
    T_k(n) + \tau_k(n) -1 - (1-\delta)\tau_k(t) \ge \frac{(n-K-(1-\delta)\tau)h_k}{H},
  \end{equation}
  where $H$ and $h_k$ are defined in \eqref{eq:H-2}. We know that such an arm
  exists, otherwise we have
  \begin{align*}
    n-K-(1-\delta)\tau
    & > \sum_{i=1}^K[T_i(n) + \tau_i(n) -1 - (1-\delta)\tau_i(t)]\\
    & \ge n- K - (1-\delta)\tau,
  \end{align*}
  which is a contradiction.

  By the definition of $t$ (via $\tilde{t}$) and \eqref{eq:Cauchy-2}, we have
  \begin{align*}
    T_k(t) + \delta\tau_k(t)
    & = T_k(t) + \tau_k(t) - (1-\delta)\tau_k(t)\\
    & = T_k(\tilde{t}) + \tau_k(\tilde{t}) - 1 -(1-\delta)\tau_k(t)\\
    & = T_k(n) + \tau_k(n) - 1 -(1-\delta)\tau_k(t)\\
    & \ge \frac{(n-K-(1-\delta)\tau)h_k}{H}\\
    & \ge 16 h_k a,
  \end{align*}
  which completes the proof.
\end{proof}

Given $t,\sigma>0$, define a function $f$ on positive real numbers:
\begin{equation}
  \label{eq:f}
  f(x;t,\sigma) := \frac{x^2}{2t} + \frac{\sigma x}{\sqrt{t}}.
\end{equation}

Apparently, the function $f$ is strictly increasing for $x>0$, and its inverse
function is
\begin{equation}
  \label{eq:f-1}
  f^{-1}(x;t,\sigma) = \sqrt{t}(\sqrt{\sigma^2 + 2x} - \sigma).
\end{equation}
Moreover, we derive the following property of $f^{-1}$, which will be used later.
\begin{lemma}
  \label{lem:f-1}
  Fix $t,\sigma>0$, $f^{-1}$ is increasing and sub-additive, i.e.,
  \begin{equation}
    \label{eq:f-1-subadditive}
    f^{-1}(x+y;t,\sigma) \le f^{-1}(x;t,\sigma) + f^{-1}(y;t,\sigma).
  \end{equation}
\end{lemma}
\begin{proof}
  It suffices to show that
  \[
    \sqrt{\sigma^2 + 2(x+y)} + \sigma \le \sqrt{\sigma^2 + 2x} + \sqrt{\sigma^2
      + 2y}.
  \]
  In fact, a simple calculation yields
  \begin{align*}
    (\sqrt{\sigma^2 + 2(x+y)} + \sigma)^2
    & = 2\sigma^2 + 2(x+y) + 2 \sigma \sqrt{\sigma^2 + 2(x+y)}\\
    & \le 2\sigma^2 + 2(x+y) + 2 \sqrt{\sigma^2 + 2x} \sqrt{\sigma^2 + 2y}\\
    & = (\sqrt{\sigma^2 + 2x} + \sqrt{\sigma^2
      + 2y})^2,
  \end{align*}
  which completes the proof.
\end{proof}
Now it's ready to prove the following lemma:
\begin{lemma}
  \label{lem:main-inq-2}
  Under the same assumptions as Proposition \ref{prop:general-2}, the inequality
  \eqref{eq:main-inq-2} holds on event $E := E_1\cap E_2$, where $E_1$ and $E_2$
  are defined in \eqref{eq:E1} and
  \eqref{eq:E2}, respectively.
\end{lemma}

\begin{proof}
  Choose arm $k$ satisfying \eqref{eq:arm_k-2} in Lemma \ref{lem:arm_k-2}.

  Since $I_{t+1} = k$, we have for each $i\in [K]$,
  \begin{equation}
    \label{eq:BiBk-tau}
    B_k(t) \le B_i(t).
  \end{equation}
  By triangle inequality, we have for every $i\in [K]$,
  \[
    |\hat{\Delta}_{i,T_i(t)} - \Delta_i| \le |\hat{\mu}_{i,T_i(t)}-\mu_i|.
  \]
  Therefore, on event $E_1$, we have
  \begin{equation}
  \label{eq:deltaEst}
  |\hat{\Delta}_{i,T_i(t)} - \Delta_i| \le f(\sqrt{a}; T_i(t), \hat{\sigma}_{i,T_i(t)}),
\end{equation}
where $f$ is defined in \eqref{eq:f}.

Now we try to find an upper bound for $B_i(t)$. It follows from Lemma
\ref{lem:f-1} and \eqref{eq:deltaEst} that
\begin{align*}
   B_i(t)
  & =  \sqrt{T_i(t)+\delta \tau_i(t)}\left(\sqrt{\hat{\sigma}^2_{i,T_i(t)}+ \hat{\Delta}_{i,T_i(t)}} - \hat{\sigma}_{i,T_i(t)}\right)\\
  & = f^{-1}(\hat{\Delta}_{i,T_i(t)}/2;T_i(t),\hat{\sigma}_{i,T_i(t)})\cdot \sqrt{\frac{T_i(t)+\delta\tau_i(t)}{T_i(t)}}\\
  & \le \sqrt{1+\delta\eta} f^{-1}(\hat{\Delta}_{i,T_i(t)}/2;T_i(t),\hat{\sigma}_{i,T_i(t)})\\
  & \le \sqrt{1+\delta\eta} f^{-1}(\Delta_i/2 + f(\sqrt{a}; T_i(t), \hat{\sigma}_{i,T_i(t)})/2;T_i(t),\hat{\sigma}_{i,T_i(t)})\\
           & \le \sqrt{1+\delta\eta} \left[f^{-1}(\Delta_i/2;T_i(t),\hat{\sigma}_{i,T_i(t)}) + f^{-1}(f(\sqrt{a}; T_i(t), \hat{\sigma}_{i,T_i(t)})/2; T_i(t),\hat{\sigma}_{i,T_i(t)})\right]\\
           & \le \sqrt{1+\delta\eta}f^{-1}(\Delta_i/2;T_i(t),\hat{\sigma}_{i,T_i(t)}) + \sqrt{1+\delta\eta} f^{-1}(f(\sqrt{a}; T_i(t), \hat{\sigma}_{i,T_i(t)}); T_i(t),\hat{\sigma}_{i,T_i(t)})\\
           & = \sqrt{1+\delta\eta} f^{-1}(\Delta_i/2; T_i(t), \hat{\sigma}_{i,T_i(t)}) + \sqrt{1+\delta\eta} \sqrt{a}.
\end{align*}
Similarly, we can bound $B_k(t)$ from below:
\begin{align*}
   B_k(t)
  & =  \sqrt{T_k(t)+\delta \tau_k(t)}\left(\sqrt{\hat{\sigma}^2_{k,T_k(t)}+ \hat{\Delta}_{k,T_k(t)}} - \hat{\sigma}_{k,T_k(t)}\right)\\
  & = f^{-1}(\hat{\Delta}_{k,T_k(t)}/2;T_k(t),\hat{\sigma}_{k,T_k(t)})\cdot \sqrt{\frac{T_k(t) + \delta\tau_k(t)}{T_k(t)}}\\
  & \ge \sqrt{\frac{T_k(t) + \delta\tau_k(t)}{T_k(t)}}\bigg[f^{-1}(\Delta_k/2;T_k(t),\hat{\sigma}_{k,T_k(t)})\\
  & \qquad - f^{-1}(f(\sqrt{a}; T_k(t), \hat{\sigma}_{k,T_k(t)})/2; T_k(t),\hat{\sigma}_{k,T_k(t)})\bigg]\\
  & \ge \sqrt{\frac{T_k(t) + \delta\tau_k(t)}{T_k(t)}} f^{-1}(\Delta_k/2;T_k(t),\hat{\sigma}_{k,T_k(t)})\\
  & \qquad - \sqrt{1+\delta\eta}f^{-1}(f(\sqrt{a}; T_k(t), \hat{\sigma}_{k,T_k(t)}); T_k(t),\hat{\sigma}_{k,T_k(t)})\\
  & = \sqrt{\frac{T_k(t) + \delta\tau_k(t)}{T_k(t)}} f^{-1}(\Delta_k/2; T_k(t), \hat{\sigma}_{k,T_k(t)}) - \sqrt{1+\delta\eta}\sqrt{a}.
\end{align*}
Set $\ds\epsilon : = \sqrt{\frac{a}{T_k(t)}}$, on event $E_2$ we have
\begin{align*}
  f^{-1}(\Delta_k/2; T_k(t), \hat{\sigma}_{k,T_k(t)})
  & = \sqrt{T_k(t)} \left(\sqrt{\hat{\sigma}_{k,T_k(t)}^2 + \Delta_k} - \hat{\sigma}_{k,T_k(t)}\right)\\
  & = \frac{\sqrt{T_k(t)}\Delta_k}{\sqrt{\hat{\sigma}_{k,T_k(t)}^2 + \Delta_k} + \hat{\sigma}_{k,T_k(t)}}\\
  & \ge \frac{\sqrt{T_k(t)}\Delta_k}{\sqrt{(\sigma_k+\epsilon)^2 + \Delta_k} + \sigma_k + \epsilon}\\
  & = \sqrt{T_k(t)}\left(\sqrt{(\sigma_k+\epsilon)^2 + \Delta_k} - \sigma_k - \epsilon\right),
\end{align*}
and thus
\[
  \sqrt{\frac{T_k(t) + \delta\tau_k(t)}{T_k(t)}} f^{-1}(\Delta_k/2; T_k(t),
  \hat{\sigma}_{k,T_k(t)}) \ge \sqrt{T_k(t) +
    \delta\tau_k(t)}\left(\sqrt{\sigma_k^2 + \Delta_k} - \sigma_k\right) - \sqrt{1+\delta\eta}\sqrt{a}.
\]
To sum up, we obtain
\[
  f^{-1}(\Delta_i/2;T_i(t),\hat{\sigma}_{i,T_i(t)}) \ge \frac{\sqrt{T_k(t)+\delta\tau_k(t)}}{\sqrt{1+\delta\eta}}\left(\sqrt{\sigma_k^2 + \Delta_k} - \sigma_k\right) - 3\sqrt{a}.
\]
Then it follows from \eqref{eq:arm_k-2} in Lemma \ref{lem:arm_k-2} that
\[
  f^{-1}(\Delta_i/2;T_i(t),\hat{\sigma}_{i,T_i(t)}) \ge \sqrt{a},
\]
or equivalently,
  \[
    \sqrt{\frac{\hat{\sigma}_{i,T_i(t)}^2a}{T_i(t)}} +
    \frac{a}{2T_i(t)} \le \frac{\Delta_i}{2}.
  \]
  Since we are on event $E_2$ defined in \eqref{eq:E2}, we get
  \[
    \sqrt{\frac{a}{T_i(t)}}\left(\sigma_i - \sqrt{\frac{a}{32T_i(t)}}\right) +
    \frac{a}{2T_i(t)} \le \frac{\Delta_i}{2},
  \]
  and thus
  \[
    \sqrt{\frac{a}{T_i(t)}}\sigma_i  +
    \left(\frac{1}{2}-\frac{1}{4\sqrt{2}}\right) \frac{a}{T_i(t)} \le \frac{\Delta_i}{2}.
  \]
  Using the fact that $T_i(t)\le T_i(n)$, we have
  \[
    \sqrt{\frac{a}{T_i(n)}}\sigma_i  +
    \left(\frac{1}{2}-\frac{1}{4\sqrt{2}}\right) \frac{a}{T_i(n)} \le \frac{\Delta_i}{2}.
  \]
  Now we use the estimate
  \[
    \hat{\sigma}_{i,T_i(n)} \le \sigma_i + \sqrt{\frac{a}{32T_i(n)}},
  \]
  due to the event $E_2$, and obtain
  \[
    \sqrt{\frac{\hat{\sigma}_{i,T_i(n)}^2a}{T_i(n)}}+  \frac{(2-\sqrt{2})a}{4T_i(n)} \le \frac{\Delta_i}{2},
  \]
  which is exactly \eqref{eq:main-inq-2}, and completes the proof. 
\end{proof}

\paragraph{Proof of Theorem 4}
\begin{proof}
 It is a direct application of Proposition \ref{prop:general-2} and similar to
the proof of Theorem \ref{thm:AP-EVT}, and we omit the proof here.
\end{proof}

\paragraph{Lower bound}
In this part, we present a lower bound of probability of making a mistake for
the thresholding bandit problem. More specifically, given a set of bandit
settings, all with the same problem complexity $H_{\text{EVT}}$, our lower bound
shows that any algorithm, however good it is, will make a mistake with
probability, which matches the upper bound in \eqref{eq:upperbound}, up to some constants.

To this end, let us introduce some notations. 
Set $b = 1/2$, and let $0\le \Delta_k\le 1/4$ for all $k\in [K]$. Let us write
for any $k\in [K]$,
\[
  \nu_k \sim \text{Ber}(b + \Delta_k), \quad\text{and}\quad \nu_k'\sim \text{Ber}(b-\Delta_k).
\]
We define the product distribution $\cB^i$ as $\nu_1^i\otimes\cdots
\otimes\nu_k^i$, $1\le i\le K$, where
\[
  \nu_k^i : = \nu_i\mathbbm{1}(k\neq i) + \nu_i'\mathbbm{1}(k=i).
\]
By the thresholding bandit problem $i$, written as TBP$(i)$, we mean for any
$k\in [K]$, arm $k$ has distribution $\nu_k^i$, that is, all arms except arm $i$
are above the threshold. In the same spirit, the thresholding bandit problem
TBP$(0)$, associated with the distribution $\cB^0 :=
\nu_1\otimes\cdots\otimes\nu_K$ is such that all arms are above the threshold.

We simply write $\bP_i:= \bP_{(\cB^i)^{\otimes T}}$, which is the probability
distribution of TBP$(i)$ according to all the samples that an algorithm could
possibly collect up to horizon $n$ , i.e. according to the samples
$\{X_{k,s}\}_{k\in [K], 1\le s\le n}$.

It is apparent that the gaps between the arm and the threshold are
$\{\Delta_k\}_{k\in [K]}$ for all these $K+1$ problems.

Then our Theorem \ref{thm:lowerbound} can be rephrased as follows: 
  for any thresholding bandit algorithm,
  \[
    \max_{i\in [K]} \E_i(\cL(n)) \ge \exp(-10n/H_{\text{EVT}}- 16\log(5nK)),
  \]
  where $\cL(n)$ is defined in \eqref{eq:L} and $\E_i$ is the expectation under the probability distribution $\bP_i$.

To prove the lower bound, we first introduce a lemma.
\begin{lemma}
  \label{lem:compareH}
  Assume $X$ is a random variable bounded in $[0,1]$. Let $\mu: = \E X$ and
  $\sigma^2 := \E(X-\mu)^2$. For any $b\in [0,1]$, we have
  \[
    \sigma^2 + |\mu - b|\le 1.
  \]
  \begin{proof}
    Since $X\in [0,1]$, we have $X^2\le X$. Thus,
    \[
      \sigma^2 = \E X^2 - (\E X)^2 \le \E X - (\E X)^2 = \mu - \mu^2.
    \]
    If $\mu< b$,
    \[
      |\mu-b| = b-\mu \le 1 + \mu^2 -\mu.
    \]
    If $\mu \ge b$,
    \[
      |\mu-b| = \mu-b \le \mu \le 1 + \mu^2 -\mu.
    \]
  \end{proof}
\end{lemma}
\textbf{Step 1.} The Kullback-Leibler divergence between the two Bernoulli
  distributions $\nu_k$ and $\nu_k'$ is
  \[
    \text{KL}(\nu_k,\nu_k') = \text{KL}(\nu_k', \nu_k) := \text{KL}_k = \Delta_k \log
    \frac{1/2 + \Delta_k}{1/2- \Delta_k}.
  \]
  Note that $\Delta_k\le 1/4$, we have
  \begin{equation}
    \label{eq:KLbound}
    \text{KL}_k \le 8 \Delta_k^2.
  \end{equation}
  Let $1\le t\le n$, we define the quantity
  \begin{align*}
    \widehat{\text{KL}}_{k,t}
    & = \frac{1}{t} \sum_{s=1}^t \log\left(\frac{d\nu_k}{d\nu_k'}(X_{k,s})\right)\\
    & = \frac{1}{t} \sum_{s=1}^t \log \frac{(1/2 + \Delta_k)X_{k,s}+ (1/2-\Delta_k)(1-X_{k,s})}{(1/2-\Delta_k)X_{k,s} + (1/2+ \Delta_k)(1-X_{k,s})},
  \end{align*}
  or equivalently,
\[
  \widehat{\text{KL}}_{k,t} = \frac{1}{t} \sum_{s=1}^t \mathbbm{1}(X_{k,s} =
  1)\log \frac{1/2 + \Delta_k}{1/2-\Delta_k} + \mathbbm{1}(X_{k,s}=0)\log \frac{1/2-\Delta_k}{1/2+\Delta_k}.
\]
Next, we will show that the following event
\[
  \cE = \left\{\forall k\in [K], \forall 1\le t\le n, \widehat{\text{KL}}_{k,t}
    - \text{KL}_k \le 8\sqrt{2}\Delta_k \sqrt{\frac{\log(4nK)}{t}}\right\},
\]
happens with high probability for all $i\in [K]$, and more specifically,
\begin{equation}
  \label{eq:probE}
  \bP_i(\cE) \ge 3/4.
\end{equation}
To this end, we define
\[
  \widetilde{\text{KL}}_{k,t} = \frac{1}{t} \sum_{s=1}^t \left|\mathbbm{1}(X_{k,s} =
  1)\log \frac{1/2 + \Delta_k}{1/2-\Delta_k} + \mathbbm{1}(X_{k,s}=0)\log \frac{1/2-\Delta_k}{1/2+\Delta_k}\right|.
\]
Then we have $\E_i \widetilde{\text{KL}}_{k,t} = \text{KL}_k$. 
Moreover,
\[
 \left|\log\left(\frac{d\nu_k}{d\nu_k'}(X_{k,s})\right)\right| = \left|\mathbbm{1}(X_{k,s} =
  1)\log \frac{1/2 + \Delta_k}{1/2-\Delta_k} + \mathbbm{1}(X_{k,s}=0)\log
  \frac{1/2-\Delta_k}{1/2+\Delta_k}\right| \le 8 \Delta_k,
\]
since $\Delta_k\le 1/4$. Therefore, $\widetilde{\text{KL}}_{k,t}$ is the sample
mean of independent random variables that are bounded by $8\Delta_k$, and thus
by Hoeffding's inequality we obtain
\[
  \widetilde{\text{KL}}_{k,t} - \text{KL}_k \le 8\sqrt{2}\Delta_k \sqrt{\frac{\log(4nK)}{t}},
\]
with probability at least $1-(4nK)^{-1}$. Since $\widehat{\text{KL}}_{k,t}\le
\widetilde{\text{KL}}_{k,t}$, we arrive at \eqref{eq:probE} by the union bound
over all $k\in [K]$ and $1\le t\le n$.

\textbf{Step 2: A change of measure.} Let $\cA_i = \{i\in \hat{U}_b(n)\}$ be the
event that the algorithm identified arm $i$ to be above the threshold. In other words, the algorithm made at least one mistake on
event $\cA_i$, particularly on arm $i$. By a change of measure between $\cB^i$
and $\cB^0$ we have
\begin{align*}
  \bP_i(\cA_i)
  & = \E_0\left[\mathbbm{1}_{\cA_i} \exp\left(-T_i(n) \widehat{\text{KL}}_{i,T_i(n)}\right)\right]\\
  & \ge \E_0 \left[ \mathbbm{1}_{\cA_i\cap \cE} \exp\left(-T_i(n) \widehat{\text{KL}}_{i,T_i(n)}\right) \right]\\
  & \ge \E_0 \left[ \mathbbm{1}_{\cA_i\cap \cE} \exp\left(-8\Delta_i^2 T_i(n) - 8\sqrt{2}\Delta_i \sqrt{T_i(n)}\sqrt{\log(4nK)}\right)\right]\\
  & \ge \E_0 \left[ \mathbbm{1}_{\cA_i\cap \cE} \exp\left(-10\Delta_i^2 T_i(n) - 16\log(4nK)\right)\right].
\end{align*}
Now set $\cA=\cap_{i\in [K]}\cA_i$, i.e., the event that all arms are identified
as being above the threshold. Then
\begin{align*}
  \max_{i\in [K]}\bP_i(\cA_i)
  & \ge \frac{1}{K} \sum_{i=1}^K \bP_i(\cA_i)\\
  & \ge \frac{1}{K} \sum_{i=1}^K \E_0 \left[ \mathbbm{1}_{\cA_i\cap \cE} \exp\left(-10\Delta_i^2 T_i(n) - 16\log(4nK)\right)\right]\\
  & \ge \E_0\left[ \mathbbm{1}_{\cA\cap \cE} \frac{1}{K} \sum_{i=1}^K  \exp\left(-10\Delta_i^2 T_i(n) - 16\log(4nK)\right)\right]\\
  & \ge \exp(-16\log(4nK)) \E_0 \left[\mathbbm{1}_{\cA\cap \cE} S\right],
\end{align*}
where
\[
  S = \frac{1}{K}\sum_{i=1}^K \exp(-10 h_i^{-2} T_i(n)).
\]
with
\[
  h_i = \sigma_i^2 \Delta_i^{-2} + \Delta_i^{-1}.
\]
Here we use the fact that $h_i \le \Delta_i^{-2}$ by Lemma \ref{lem:compareH}. 

It follows from the identity $\sum_{i=1}^K T_i(n) = n$
that we have
\[
  n = \sum_{i=1}^K T_i(n) h_i^{-2} h_i^2 \ge \min_{i\in [K]} T_i(n)h_i^{-2}
  \sum_{i=1}^K h_i^2 = \min_{i\in [K]} T_i(n)h_i^{-2}
  H_{\text{EVT}}.
\]
Therefore,
\[
  S \ge \frac{1}{K} \exp\left(-10 \min_{i\in [K]} T_i(n)h_i^{-2}\right) \ge
  \frac{1}{K}\exp\left(-\frac{10 n}{H_{\text{EVT}}}\right) =
  \exp\left(-\frac{10n}{H_{\text{EVT}}}- \log K\right),
\]
and thus
\[
  \max_{i\in [K]}\bP_i(\cA_i) \ge
  \exp\left(-\frac{10n}{H_{\text{EVT}}}-16\log(4nK) -\log K\right) \bP_0(\cA\cap
  \cE).
\]

\textbf{Step 3.} To finish the proof, we first consider the case that
$\bP_0(\cA) \le 1/2$. So we have 
\begin{align*}
  \max_{i\in [K]\cup\{0\}} \E_i(\cL(n))
  & \ge \E_0(\cL(n)) = \bP_0(\cup_{i\in [K]} \cA_i^c) = 1- \bP_0(\cA) \ge 1/2.
\end{align*}
In the case that $\bP_0(\cA) \ge 1/2$, we have
\[
  \bP_0(\cA\cap \cE) \ge \bP_0(\cA) + \bP_0(\cE) - 1 \ge 1/2+3/4-1 = 1/4,
\]
by \eqref{eq:probE}. Therefore,
\begin{align*}
  \max_{i\in [K]\cup\{0\}} \E_i(\cL(n))
  & \ge \max_{i\in [K]} \E_i(\cL(n)) \ge \frac{1}{4}\exp\left(-\frac{10n}{H_{\text{EVT}}}-16\log(4nK) -\log K\right)\\
  & \ge \exp\left(-\frac{10n}{H_{\text{EVT}}}-16\log(5nK)\right).
\end{align*}
To sum up, we have
\begin{align*}
  \max_{i\in [K]\cup\{0\}} \E_i(\cL(n)) \ge
  & \min\left\{1/2, \exp\left(-\frac{10n}{H_{\text{EVT}}}-16\log(5nK)\right)\right\}\\
  & = \exp\left(-\frac{10n}{H_{\text{EVT}}}-16\log(5nK)\right),
\end{align*}
which completes the proof.

\end{document}